\documentclass{article}

\usepackage{aaai21}  

\usepackage[comma,authoryear]{natbib}

\usepackage{hyperref}
\usepackage{microtype}
\usepackage{graphicx}
\usepackage{subfigure}
\usepackage{booktabs} 
\usepackage{todonotes}
\usepackage{bbm}
\usepackage{hyperref}
\usepackage{wrapfig} 
\usepackage{todonotes}
\usepackage{microtype}
\usepackage{graphicx}
\usepackage{booktabs} 
\usepackage{wrapfig}
\usepackage{algorithm}
\usepackage[noend]{algorithmic}
\usepackage{tabularx}
\usepackage{amsmath}
\usepackage{amssymb}
\usepackage{mathrsfs}
\usepackage{amsthm}
\usepackage{stmaryrd}
\usepackage{mathtools}

\DeclareTextFontCommand{\emph}{\em}
\usepackage{bm}
\usepackage{dblfloatfix}
\usepackage{enumitem}
\usepackage[thinc]{esdiff}
\usepackage{bigints}
\usepackage{mdwlist}
\usepackage{thmtools, thm-restate}
\newtheorem{lemma}{Lemma}
\newtheorem{lemmaap}{Lemma}

\declaretheorem{theorem}

\declaretheorem{corollary}
\declaretheorem{definition}
\usepackage{amsmath}
\usepackage{centernot}

\newtheorem{innercustomthm}{Theorem}
\newenvironment{customthm}[1]
  {\renewcommand\theinnercustomthm{#1}\innercustomthm}
  {\endinnercustomthm}

\usepackage[switch]{lineno}

\usepackage{tikz-cd}

\newcommand{\cA}{\mathcal{A}}

\newcommand{\cF}{\mathcal{F}}

\newcommand{\cM}{\mathcal{M}}

\newcommand{\cP}{\mathcal{P}}
\newcommand{\cR}{\mathcal{R}}
\newcommand{\cS}{\mathcal{S}}
\newcommand{\cT}{\mathcal{T}}
\newcommand{\cW}{\mathcal{W}}
\newcommand{\rE}{\mathbb{E}}
\newcommand{\rI}{\mathbb{I}}
\newcommand{\rM}{\mathbb{M}}

\newcommand{\rR}{\mathbb{R}}
\newcommand{\rT}{\mathbb{T}}

\newcommand{\mdp}{\mathcal{M}}
\newcommand{\sspace}{\mathcal{S}}
\newcommand{\rew}{\mathcal{R}}
\newcommand{\prob}{\mathcal{P}}
\newcommand{\trans}{\mathcal{T}}
\newcommand{\aspace}{\mathcal{A}}

\newcommand{\cbar}{\, | \,}

\newcommand{\csemi}{\, ; \,}

\title{Metrics and continuity in reinforcement learning}

\author{
  Charline Le Lan  \thanks{Work performed while a Google Student Researcher.} \textsuperscript{\rm 1}, Marc G. Bellemare\textsuperscript{\rm 2}, Pablo Samuel Castro\textsuperscript{\rm 2}\\}
  \affiliations{
  \textsuperscript{\rm 1}University of Oxford, \textsuperscript{\rm 2}Google Research, Brain Team \\
  \texttt{charline.lelan@stats.ox.ac.uk }, \texttt{\{bellemare,psc\}@google.com}
}

\begin{document}

\maketitle

\begin{abstract}
  In most practical applications of reinforcement learning, it is untenable to maintain direct estimates for individual states; in continuous-state systems, it is impossible. Instead, researchers often leverage {\em state similarity} (whether explicitly or implicitly) to build models that can generalize well from a limited set of samples. The notion of state similarity used, and the neighbourhoods and topologies they induce, is thus of crucial importance, as it will directly affect the performance of the algorithms. Indeed, a number of recent works introduce algorithms assuming the existence of ``well-behaved'' neighbourhoods, but leave the full specification of such topologies for future work. In this paper we introduce a unified formalism for defining these topologies through the lens of metrics. We establish a hierarchy amongst these metrics and demonstrate their theoretical implications on the Markov Decision Process specifying the reinforcement learning problem. We complement our theoretical results with empirical evaluations showcasing the differences between the metrics considered.
\end{abstract}

\section{Introduction}
A simple principle to generalization in reinforcement learning is to require that similar states be assigned similar predictions.
State aggregation implements a coarse version of this principle, by using a notion of similarity to group states together.
A finer implementation is to use the similarity in an adaptive fashion, for example by means of a nearest neighbour scheme over representative states.
This approach is classically employed in the design of algorithms for continuous state spaces, where the fundamental assumption is the existence of a metric characterizing the real-valued distance between states.

\begin{figure}[t!]
  \centering
  \includegraphics[width=0.3\textwidth]{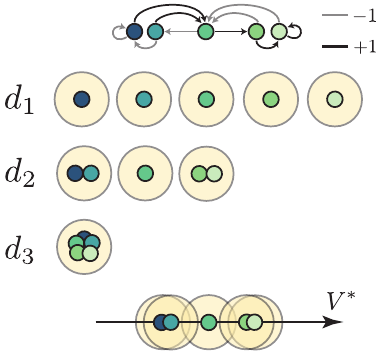}
  \caption{A simple five-state MDP (top) with the neighbourhoods induced by three metrics: an identity metric which isolates each state ($d_1$); a metric which captures behavioral proximity ($d_2$); and a metric which is not able to distinguish states ($d_3$). The yellow circles represent $\epsilon$-balls in the corresponding metric spaces. The bottom row indicates the $V^*$ values for each state.}
  \label{fig:highLevelFigure}
\end{figure}

To illustrate this idea, consider the three similarity metrics depicted in \autoref{fig:highLevelFigure}. The metric $d_1$ isolates each state, the metric $d_3$ groups together all states, while the metric $d_2$ aggregates states based on the similarity in their {\em long-term dynamics}. In terms of generalization, $d_1$ would not be expected to generalize well as new states cannot leverage knowledge from previous states; $d_3$ can cheaply generalize to new states, but at the expense of accuracy; on the other hand, $d_2$ seems to strike a good balance between the two extremes.

In this paper we study the effectiveness of \emph{behavioural metrics} at providing a good notion of state similarity.
We call behavioural metrics the class of metrics derived from properties of the environment, typically measuring differences in reward and transition functions.
Since the introduction of bisimulation metrics \citep{ferns2004metrics,ferns05metrics}, a number of behavioural metrics have emerged with additional desirable properties, including lax bisimulation \citep{taylor2009bounding,castro2010using} and $\pi$-bisimulation metrics \citep{castro2020scalable}.
Behavioural metrics are of particular interest in the context of understanding generalization, since they directly encode the differences in action-conditional outcomes between states, and hence allow us to make meaningful statements about the relationship between these states.

We focus on the interplay between behavioural metrics and the continuity properties they induce on various functions of interest in reinforcement learning. Returning to our example, $V^*$ is only continuous with respect to $d_1$ and $d_2$.
The continuity of a set of functions (with respect to a given metric) is assumed in most theoretical results for continuous state spaces,
such as uniform continuity of the transition function \citep{kakade2003exploration}; Lipschitz continuity of all Q-functions of policies \citep{pazis2013pac}, Lipschitz continuity of the rewards and transitions \citep{zhao2014mec,ok2018exploration} or of the optimal Q-function \citep{song2019efficient,touati2020zooming,sinclair2019adaptive}.
We find that behavioural metrics support these algorithms to varying degrees: the original bisimulation metric, for example, provides fewer guarantees  than what is required by some near-optimal exploration algorithms \citep{pazis2013pac}.
These results are particularly significant given that behavioural metrics form a relatively privileged group: any metric that enables generalization must in some sense reflect the structure of interactions within the environment and hence, act like a behavioural metric.

\section{Overview}

Our aim is to unify representations of state spaces and the notion of continuity via a taxonomy of metrics.

Our first contribution is a general result about the continuity relationships of different functions of the MDP (Theorem 1). While \citet{gelada2019deepmdp} (resp. \citet{norets2010continuity}) proved the uniform Lipschitz continuity of the optimal action-value function (resp. local continuity of the optimal value function) given the uniform Lipschitz continuity (resp. local continuity) of the reward and transition functions and \citet{rachelson2010locality} showed the uniform Lipschitz continuity of the value function given the uniform Lipschitz continuity of the action-value function in the case of deterministic policies, Theorem 1 is a more comprehensive result about the different components of the MDP (reward and transition functions, value and action value functions), for a spectrum of continuity notions (local and uniform continuity, local and uniform Lipschitz continuity) and applicable with stochastic policies, also providing counterexamples demonstrating that these relationships are only implication results. 

Our second contribution is to demonstrate that different metrics lead to different notions of continuity for different classes of functions (Section Continuity: Prior metrics, Section value-based metrics and Table 2). We first study metrics that have been introduced in the literature (presented in Section Prior metrics and abstractions). While \citet{li2006towards} provide a unified treatment of some of these metrics, they do not analyse these abstractions through the lens of continuity. Using our taxonomy, we find that most commonly discussed metrics are actually poorly suited for algorithms that convert representations into values, so we introduce new metrics to overcome this shortcoming (section Value-based metrics). We also analyse the relationships between the topologies induced by all the metrics in our taxonomy (Theorem 2). 

Finally, we present an empirical evaluation that supports our taxonomy and shows the importance of
the choice of a neighbourhood in reinforcement learning algorithms (section Empirical evaluation).

\section{Background}
\label{sec:background}

We consider an agent interacting with an environment, modelled as a Markov Decision Process (MDP) $\mdp = \langle\sspace, \aspace, \rew, \prob, \gamma\rangle$ \citep{puterman1994markov}. Here $\sspace$ is a continuous state space with Borel $\sigma$-algebra $\Sigma$ and $\aspace$ a discrete set of actions. Denoting $\Delta(X)$ to mean the probability distribution over $X$, we also have that $\prob : \sspace \times \aspace \to \Delta(\sspace)$ is the transition function, $\rew : \sspace \times \aspace \rightarrow [0, R_{\text{max}}]$ is the measurable reward function, and $\gamma \in [0, 1)$ is the discount factor. We write $\prob_s^a$ to denote the next-state distribution over $\sspace$ resulting from selecting action $a$ in $s$ and write $\rew_s^a$ for the corresponding reward.

A stationary policy $\pi:\sspace \rightarrow \Delta(\aspace)$ is a mapping from states to distributions over actions, describing a particular way of interacting with the environment. We denote the set of all policies by $\Pi$. For any policy $\pi \in \Pi$, the value function $V^\pi(s)$ measures the expected discounted sum of rewards received when starting from state $s\in\sspace$ and acting according to $\pi$:
\begin{equation*}
	V^\pi(s) := \mathop\mathbb{E} \Big [ \sum_{t \ge 0} \gamma^t \rew_{s_t}^{a_t} \csemi s_0 = s, a_t \sim \pi(\cdot \cbar s_t) \Big ] .
\end{equation*}
The maximum attainable value is $V_{\text{max}} := \frac{R_{\text{max}}}{1-\gamma}$.
The value function satisfies Bellman's equation:
\begin{equation*}
	V^{\pi}(s) = \mathop\mathbb{E}_{a\sim\pi(\cdot\cbar s)}[\rew_s^a + \gamma \mathop\mathbb{E}_{s'\sim\prob_s^a}V^{\pi}(s')] .
\end{equation*}
The state-action value function or Q-function $Q^\pi$ describes the expected discounted sum of rewards when action $a \in \aspace$ is selected from the starting state $s$, and satisfies the recurrence
\begin{equation*}
	Q^{\pi}(s, a) = \rew_s^a + \gamma \mathop\mathbb{E}_{s'\sim\prob_s^a} V^{\pi}(s').
\end{equation*}
A policy $\pi$ is said to be optimal if it maximizes the value function at all states: 
\begin{equation*}
	V^{\pi}(s) = \mathop{\max}_{\pi' \in\Pi}V^{\pi'}(s) \text{ for all } s \in \sspace .
\end{equation*}
The existence of an optimal policy is guaranteed in both finite and infinite state spaces. We will denote this policy $\pi^* \in \Pi$. The corresponding value function and Q-function are denoted respectively $V^*$ and $Q^*$.

\subsection{Metrics, topologies, and continuity}

We begin by recalling standard definitions regarding metrics and continuity, two concepts central to our work.
\begin{definition}[\citeauthor{royden68real}, \citeyear{royden68real}]
  \label{def:metricspace}
  A \textbf{metric space} $\langle X, d\rangle$ is a nonempty set $X$ of elements (called points) together with a real-valued function $d$ defined on $X\times X$ such that for all $x$, $y$, and $z$ in $X$: $d(x, y)\geq 0$; $d(x, y) = 0$ if and only if $x = y$; $d(x, y) = d(y, x)$ and $d(x, y)\leq d(x, z) + d(z, y)$.
  The function $d$ is called a \textbf{metric}.
  A \textbf{pseudo-metric} $d$ is a metric with the second condition replaced by the weaker condition $x=y\implies d(x, y)=0$.
\end{definition}
In what follows, we will often use {\em metric} to stand for {\em pseudo-metric} for brevity.

A metric $d$ is useful for our purpose as it quantifies, in a real-valued sense, the relationship between states of the environment. Given a state $s$, a natural question is: What other states are similar to it? The notion of a \emph{topology} gives a formal answer.
\begin{definition}[\citeauthor{sutherland2009introduction}, \citeyear{sutherland2009introduction}]
  A metric space $\langle X, d\rangle$ induces a \textbf{topology} $(X, \mathcal{T}_d)$ defined as the collection of open subsets of $X$; specifically, the subsets $U \subset X$ that satisfy the property that for each $x \in U$, there exists $\epsilon >0$ such that the \textbf{$\epsilon$-neighbourhood} $B_d(x, \epsilon)=\{y \in X| d(y, x)<\epsilon\} \subset U$.\\
  Let $(X, \mathcal{T})$ and $(X, \mathcal{T'})$ be two topologies on the same space $X$. We say that $\mathcal{T}$ is \textbf{coarser} than $\mathcal{T'}$, or equivalently that $\mathcal{T'}$ is \textbf{finer} than $\mathcal{T}$, if $\mathcal{T} \subset \mathcal{T'}$.
\end{definition}

Given two similar states under a metric $d$, we are interested in knowing how functions of these states behave. In the introductory example, we asked specifically: how does the optimal value function behave for similar states? This leads us to the notion of functional continuity.
Given $f: X \rightarrow Y$ a function between a metric space $(X, d_X)$ and a metric space $(Y, d_Y)$,
  \begin{itemize}[leftmargin=*]
    \item \textbf{Local continuity (LC)}: $f$ is locally continuous at $x\in X$ if for any $\epsilon >0$, there exists a $\delta_{x,\epsilon} > 0$ such that for all $x' \in X$, $d_X(x, x') < \delta_{x, \epsilon} \implies d_Y(f(x), f(x')) < \epsilon$. $f$ is said to be locally continuous on $X$ if it is continuous at every point $x\in X$.
    \item \textbf{Uniform continuity (UC)}: $f$ is uniformly continuous on $X$ when given any $\epsilon > 0$, there exists $\delta_\epsilon > 0$ such that for all $x, x' \in X, d_X(x, x') < \delta_\epsilon \implies d_Y(f(x), f(x')) < \epsilon$.
    \item \textbf{Local Lipschitz continuity (LLC)}: $f$ is locally Lipschitz continuous at $x \in X$ if there exists $\delta_{x} > 0, K_{x}>0$ such that for all $x', x'' \in B_{d_X}(x, \delta_x)$, $d_Y(f(x'), f(x'')) \leq K_{x}d_X(x', x'')$.
    \item \textbf{Uniform Lipschitz continuity (ULC)}: $f$ is uniformly Lipschitz continuous if there exist $K>0$ such that for all $x, x' \in X$ we have $d_Y(f(x), f(x')) \leq Kd_X(x, x')$.
  \end{itemize}

The relationship between these different forms of continuity is summarized by the following diagram:
  \begin{equation}
    \begin{tikzcd}
   UC \arrow[d] & ULC \arrow[d] \arrow[l] \\
   LC & LLC \arrow[l]
    \end{tikzcd}
    \label{diag:continuity}
  \end{equation}
where an arrow indicates implication; for example, any function that is ULC is also UC.

Here, we are interested in functions of states and state-action pairs. Knowing whether a particular function $f$ possesses some continuity property $p$ under a metric $d$ informs us on how well we can extrapolate the value $f(s)$ to other states; in other words, it informs us on the generalization properties of $d$.

\subsection{Prior metrics and abstractions}
\label{sec:existingmetrics}
The simplest structure is to associate states to distinct groups, what is often called state aggregation \citep{bertsekas11approximate}. This gives rise to an {\em equivalence relation}, which we interpret as a {\em discrete pseudo-metric}, that is a metric taking a countable range of values.
\begin{definition}
  \label{def:discreteMetric}
  An equivalence relation $E\subseteq X \times X$ induces a \textbf{discrete pseudo-metric} $e^E$ where $e^E(x, x')= 0$ if $(x, y)\in E$, and $1$ otherwise.
\end{definition}

Throughout the text, we will use $e$ to denote discrete pseudo-metrics.
Two extremal examples of metrics are the \textbf{identity metric} $e^{\rI}: \sspace \times \sspace \rightarrow \{0, 1\}$, induced by the {\em identity relation} $\rI = \lbrace (s, t)\in\sspace\times\sspace | s = t\rbrace$ (e.g. $d_1$ in \autoref{fig:highLevelFigure}), and the \textbf{trivial metric} $e^{\rT}: \sspace \times \sspace \rightarrow \{0\}$ that collapses all states together (e.g. $d_3$ in \autoref{fig:highLevelFigure}).

In-between these extremes, $\eta$-abstractions \citep{li2006towards,abel2017near} are functions  $\phi:\sspace\rightarrow\hat{\sspace}$ that aggregates states which are mapped close to each other by a function $f$.
That is, given a threshold $\eta \geq 0$ and $f: \sspace \times \aspace \rightarrow \mathbb{R}$, $\phi_{f, \eta}(s)=\phi_{f, \eta}(t) \implies |f(s, a)-f(t, a)|\leq\eta.$
We list a few choices for $f$ along with the name of the abstraction we will refer to throughout this text in \autoref{tbl:absfunctions}.
\begin{table*}[!h]
  \centering
  \begin{tabular}{rl}
      \toprule
    \textbf{f} & $\phi_{f, \eta}$\\
    \midrule
$Q^*$& approximate Q function abstraction ($\eta \geq 0$) / $Q^*$-irrelavance ($\eta =0$) \\
$\rew$ and $\prob$& approximate model abstraction ($\eta \geq 0$) / Model-irrelevance ($\eta =0$) \\
$Q^{\pi}$& $Q^{\pi}$-irrelevance abstraction ($\eta=0$) \\
$\max\limits_{\aspace}Q^*$& $a^*$-irrelevance abstraction ($\eta=0$) \\
        \bottomrule
  \end{tabular}
  \caption{Different types of state abstractions.}
  \label{tbl:absfunctions}
\end{table*}

$\eta$-abstractions are defined in terms of a particular function of direct relevance to the agent. However, it is not immediately clear whether these abstractions are descriptive, and, more specifically, the kind of continuity properties they support. An alternative is to relate states based on the outcomes that arise from different choices, starting in these states. These are \emph{bisimulation relations} \citep{givan2003equivalence}.

\begin{definition}
  \label{def:bisimulation}
  An equivalence relation $E\subseteq\sspace\times\sspace$ with $\sspace_{E}$ the quotient space and $\Sigma(E)$ the $\Sigma$ measurable sets closed under $E$, if whenever $(s, t)\in E$ we have:
  \begin{itemize}[leftmargin=*]
    \item \textbf{Bisimulation relation}[\citeauthor{givan2003equivalence}, \citeyear{givan2003equivalence}].

     Behavioral indistinguishability under {\em equal actions}; namely, for any action $a\in\aspace$, $\rew_s^a =\rew_t^a$, and $\prob_s^a(X) = \prob_t^a(X)$ for all $X\in\Sigma(E)$.
      We call $E$ a {\bf bisimulation relation}. We denote the largest bisimulation relation as $\sim$, and its corresponding discrete metric as $e^{\sim}$.
    \item {\bf Lax-bisimulation relation} [\citeauthor{taylor2009bounding}, \citeyear{taylor2009bounding}].

    Behavioral indistinguishability under {\em matching actions}; namely, for any action $a\in\aspace$ from state $s$ there is an action $b\in\aspace$ from state $t$ such that $\cR_s^a =\cR_t^b$, and $\prob_s^a(X) = \prob_t^b(X)$ for all $X\in\Sigma(E)$, and vice-versa, we call $E$ a {\bf lax-bisimulation relation}. We denote the largest lax-bisimulation relation as $\sim_{lax}$, and its corresponding discrete metric as $e^{\sim_{lax}}$.

    \item {\bf $\pi$-bisimulation relation} [\citeauthor{castro2020scalable}, \citeyear{castro2020scalable}].
    Behavioral indistinguishability under a {\em fixed policy}; namely, given a policy $\pi\in\Pi$, $\sum_{a\in\aspace}\pi(a | s) \cR_s^a =  \sum_{a\in\aspace}\pi(a | t) \cR_t^a$, and $\sum_{a\in\aspace}\pi(a | s) \prob_s^a(X) = \sum_{a\in\aspace}\pi(a | s)\prob_t^b(X)$ for all $X\in\Sigma(E)$. We call $E$ a {\bf $\pi$-bisimulation relation}. We denote the largest bisimulation relation as $\sim_{\pi}$, and its corresponding discrete metric as $e^{\sim_{\pi}}$.
  \end{itemize}
\end{definition}
A {\em bisimulation metric} is the continous generalization of a bisimulation relation. Formally, $d$ is a bisimulation metric if its kernel is equivalent to the bisimulation relation. The canonical bisimulation metric \citep{ferns05metrics} is constructed from the Wasserstein distance between probability distributions.
\begin{definition}
Let $(Y, d_Y)$ be a metric space with Borel $\sigma$-algebra $\Sigma$.
  The Wasserstein distance \citep{villani2008optimal} between two probability measures $P$ and $Q$ on $Y$, under a given metric $d_Y$ is given by
  $W_{d_Y}(P, Q) = \inf_{\lambda \in \Gamma(P, Q)}\mathop\mathbb{E}_{(x, y)\sim \lambda}[d_Y(x, y)]$, where $\Gamma(P, Q)$ is the set of couplings between $P$ and $Q$.
\end{definition}

\begin{lemma}[\citeauthor{ferns05metrics}, \citeyear{ferns05metrics}]
  Let $\cM$ be the space of state pseudo-metrics and define the functional $F:\cM\rightarrow\cM$ as $F(d)(x, y) = \max_{a\in\aspace}\left(|\rew^a_x - \rew^a_y| + \gamma W_d(\prob^a_x, \prob^a_y)\right)$.
  Then $F$ has a least fixed point $d^{\sim}$ and $d^{\sim}$ is a bisimulation metric.
  \label{lemmacontbisim}
\end{lemma}
In words, bisimulation metrics arise as the fixed points of an operator on the space of pseudo-metrics.
Lax bisimulation metrics $d^{\sim_{lax}}$ and a $\pi$-bisimulation metrics $d^{\sim_{\pi}}$ can be defined in an analogous fashion; for succinctness, their formal definitions are included in the appendix.

\section{Continuity relationships}
\label{sec:contRel}
Our first result characterizes the continuity relationships between key functions of the MDP. The theorem considers different forms of continuity and relates how the continuity of one function implies another. While the particular case of uniform Lipschitz continuity of $Q^*$ (resp. local continuity of $V^*$) from $\prob + \rew$ has been remarked on before by \citet{gelada2019deepmdp} (resp. \citet{norets2010continuity}) as well as the case of uniform Lipschitz continuity of $V^{\pi}$ given the uniform Lipschitz continuity of $Q^{\pi}$ for stochastic policies $\pi$ \citep{rachelson2010locality}, to the best of our knowledge this is the first comprehensive treatment of the topic, in particular providing counterexamples.
\begin{restatable}{theorem}{contRel}
  \label{thm:contRel}
  If we decompose the Cartesian product $\sspace \times \aspace$ as: $d_{S \times A}(s, a, s', a')=d_S(s, s')+d_A(a, a')$  with $d_{\aspace}$ the identity metric, the LC, UC and LLC relationships between $\prob$, $\rew$, $V^{\pi}$, $V^*$, $Q^{\pi}$ and $Q^*$ functions are given by diagram \ref{diag:contRel}.
A directed arrow $f\rightarrow g$ indicates that function $g$ is continuous whenever $f$ is continuous. Labels on arrows indicate conditions that are necessary for that implication to hold. $\prob + \rew$ is meant to stand for both $\prob$ and $\rew$ continuity; $\pi \text{-cont}$ indicates continuity of $\pi:\sspace \rightarrow \Delta(\aspace)$. An absence of a directed arrow indicates that there exists a counter-example proving that the implication does not exist. In the ULC case, the previous relationships also hold with the following additional assumptions: $\gamma L_{\prob}<1$ for $\prob + \rew \rightarrow Q^*$ and $\gamma L_{\prob}(1+L_{\pi})<1$ for $\prob + \rew \xrightarrow{\pi \text{-cont}} Q^{\pi}$ where $L_{\prob}$ and $L_{\pi}$ are the Lipschitz constants of $\prob$ and $\pi$, respectively.
\end{restatable}
\begin{equation}
  \begin{tikzcd}
    & & Q^{\pi}  \arrow[rr, "\pi \text{-cont}"]& & V^{\pi} \\
    \prob + \rew \arrow[drr] \arrow[urr, "\pi \text{-cont } "]&  & & &  \\
    & & Q^* \arrow[rr]& & V^{*}
  \end{tikzcd}
    \label{diag:contRel}
\end{equation}
\begin{proof}
All proofs and counterexamples are provided in the appendix.
\end{proof}
The arrows are transitive and apply for all forms of continuity illustrated in Diagram~(\ref{diag:continuity}); for example, if we have ULC for ${Q^*}$, this implies we have LC for $V^*$. This diagram is useful when evaluating metrics as they clarify the strongest (or weakest) form of continuity one can demonstrate.
When considering deterministic policies, we can notice that the $\pi$-continuity mentioned in \autoref{thm:contRel} is very restrictive, as the following lemma shows.
\begin{lemma}
\label{lemma1}
If a deterministic policy $\pi: \sspace \rightarrow \aspace$ is continuous, $\sspace$ is connected\footnote{A connected space is topological space that cannot be represented as the union of two or more disjoint non-empty open subsets.} and $\aspace$ is discrete, then $\pi$ is globally constant.
\end{lemma}

\section{Taxonomy of metrics}
\label{sec:metricRel}

We now study how different metrics support the continuity of functions relevant to reinforcement learning and the relationship between their induced topologies. While the taxonomy we present here is of independent interest, it also provides a clear theoretical foundation on which to build results regarding metric-respecting embeddings \citep{gelada2019deepmdp,zhang2020learning}.

\subsection{Continuity: Prior metrics}
\label{sec:discreteMetrics}
We begin the exposition by considering the continuity induced by discrete metrics. These enable us to analyze the properties of some representations found in the literature. The extremes of our metric hierarchy are the identity metric $e^{\rI}$ and trivial metric $e^{\rT}$, which respectively support all and one continuous functions, and were represented by $d_1$ and $d_3$ in the introductory example.
\label{sec:ContinuityProofs}
 \begin{restatable}[Identity metric]{lemma}{idmetric}
 \label{the:DiscreteMetricCont}
 $e^{\rI}$ induces the finest topology on $\sspace$, made of all possible subsets of $\sspace$.
 Let $(Y, d_Y)$ be any metric space.
 Any function $h$ (resp. Any bounded $h$) $:(\sspace, e^{\rI}) \rightarrow (Y, d_Y)$  is LC and UC (resp. ULC).
 \end{restatable}
 \begin{restatable}[Trivial metric]{lemma}{trivialmetric}
 \label{the:DiscreteMetricContt}
$e^{\rT}$ induces the coarsest topology on $\sspace$, consisting solely of $\{\emptyset, \sspace\}$.   Let $(Y, d_Y)$ be any metric space.
 Any function $h:(\sspace, e^{\rI}) \rightarrow (Y, d_Y)$  is LC, UC and ULC iff h is constant.
 \end{restatable}

We can also construct a discrete metric from any state aggregation $\phi:\sspace\rightarrow\hat{\sspace}$ as $e^{\phi}(s, t) = e^{\rI}(\phi(s), \phi(t)) = 0$ if $\phi(s)=\phi(t)$, and $1$ otherwise.
However, as stated below, $\eta$-abstractions do not guarantee continuity except in the trivial case where $\eta = 0$.
\begin{restatable}{lemma}{agg}\label{the:AggCont2}
If $\eta=0$, then any function $f$ (resp. bounded function $f$)$: (\sspace, d_{\sspace}) \rightarrow (Y, d_Y)$ is LC and UC (resp. ULC) with respect to the pseudometric $e^{\phi_{f, \eta}}$. However, given a function $f$ and $\eta > 0$, there exists an $\eta$-abstraction $\phi_{f, \eta}$ such that $f$ is not continuous with respect to $e^{\phi_{f, \eta}}$.
\end{restatable}
Unlike the discrete metrics defined by $\eta$-abstractions, both bisimulation metrics and the metric induced by the bisimulation relation support continuity of the optimal value function.
\begin{restatable}{lemma}{bisimcont}\label{the:bisimCont}
$Q^*$ (resp. $Q^{\pi}$) is ULC with Lipschitz constant 1 with respect to $d^{\sim}$ (resp. $d^{\sim_{\pi}}$).
\end{restatable}
\begin{restatable}{corollary}{bisimcontbis}
\label{lemma:Qcont}
$Q^*$ (resp. $Q^{\pi}$) is ULC with Lipschitz constant $V_{\text{max}}$ with respect to $e^{\sim}$ (resp. $e^{\sim_{\pi}}$).
\end{restatable}
We note that \citet{ferns2004metrics} proved a weaker statement involving $V^*$ (resp. \citet{castro09notions}, $V^{\pi}$).
To summarize, metrics that are too coarse may fail to provide the requisite continuity of reinforcement learning functions. Bisimulation metrics are particularly desirable as they achieve both a certain degree of coarseness, while preserving continuity. In practice, however, \citeauthor{ferns2004metrics}'s bisimulation metric is difficult to compute and estimate, and tends to be conservative -- as long as two states can be distinguished by action sequences, bisimulation will keep them apart.
\subsection{Value-based metrics}
\label{sec:contMetrics}

As an alternative to bisimulation metrics, we consider simple metrics constructed from value functions and study their continuity offerings. These metrics are simple in that they are defined in terms of differences between values, or functions of values, at the states being compared. The last metric, $d_{\Delta_{\forall}}$, is particularly appealing as it can be approximated, as we describe below. Under this metric, all $Q$-functions are Lipschitz continuous, supporting some of the more demanding continuous-state exploration algorithms \citep{pazis2013pac}.

\begin{restatable}{lemma}{Deltametrics}\label{the:Deltametrics}For a given MDP, let $Q^\pi$ be the Q-function of policy $\pi$, and $Q^*$ the optimal Q-function. The following are continuous pseudo-metrics:
\begin{enumerate}
\item $d_{{\Delta^*}}(s, s')=\max\limits_{a\in\aspace}|Q^*(s, a)-Q^*(s', a)|$\\
\item $d_{\Delta_{\pi}}(s, s')=\max\limits_{a\in\aspace}|Q^{\pi}(s, a)-Q^{\pi}(s', a)|$\\
\item $d_{\Delta_{\forall}}(s, s')=\max\limits_{\pi\in\Pi, a\in\aspace}|Q^{\pi}(s, a)-Q^{\pi}(s', a)|$\\
\end{enumerate}

$Q^*$ (resp. $Q^{\pi}$) is ULC with Lipschitz constant 1 wrt to $d_{\Delta^*}$ (resp. $d_{\Delta^{\pi}}$). $Q^{\pi}$ is  ULC with Lipschitz constant 1 wrt to $d_{{\Delta_{\forall}}}$ for any $\pi \in \Pi$.
\end{restatable}
\begin{restatable}{remark}{avfmetrics}\label{the:avfmetric}
When $\sspace$ is finite, the number of policies to consider to compute $d_{\Delta_{\forall}}$ is finite: $d_{\Delta_{\forall}}(s, s')=\max\limits_{\pi\in\Pi, a\in\aspace}|Q^{\pi}(s, a)-Q^{\pi}(s', a)|
=\max\limits_{\pi\in \Pi_{\text{AVF}}, a\in\aspace}|Q^{\pi}(s, a)-Q^{\pi}(s', a)|$,
where ${\Pi}_{AVF}$ is the finite set of extremal policies corresponding to Adversarial Value Functions (AVFs) \citep{bellemare19geometric}.
\end{restatable}
$d_{\Delta_{\forall}}$ provides strong continuity of the value-function for all policies contrary to any other metric that has been used in the literature.  Since computing $d_{\Delta_{\forall}}$ is computationally expensive, we will approximate it by the pseudometric $d_{\widetilde{\text{AVF}}(n)}=\max\limits_{\pi\in \Pi_{\widetilde{\text{AVF}}(n)}, a\in\aspace}|Q^{\pi}(s, a)-Q^{\pi}(s', a)|$, where $\Pi_{\widetilde{\text{AVF}}(n)}$ are $n$ samples from the set of extremal policies $\Pi_{\text{AVF}}$.

\subsection{Categorizing metrics, continuity, and complexity}

\begin{table*}[!h]
  \vspace{-1em}
  \centering
  \caption{Categorization of state metrics, their continuity implications, and their complexity (when known).
  The notation $\{y\}^{\sspace}$ denotes any function $h: \sspace \rightarrow Y$ that is constant, $Y^{\sspace}$ refers to all functions $h: \sspace \rightarrow Y$. $\mathcal{B}(Y^{\sspace})$ (resp. $\mathcal{B}_L(Y^{\sspace})$ ) is a bounded (resp. locally bounded) function $h: \sspace \rightarrow Y$. ``-'' denotes an absence of LC, UC, ULC and LLC. In the complexity column, $\delta$ is the desired accuracy.} 
    \label{tab:properties}
  \begin{tabularx}{\textwidth}{@{}cccccc@{}} 
    \toprule
    \textbf{Metric} & \textbf{LC} & \textbf{UC} & \textbf{ULC} & \textbf{LLC} &\textbf{Complexity} \\
    \midrule
    Discrete metric $e^{\rI}$& $Y^{\sspace}$& $Y^{\sspace}$ & $\mathcal{B}(Y^{\sspace})$& $\mathcal{B}_L(Y^{\sspace})$& $O(|\sspace|)$\\ 
   Trivial metric $e^{\rT}$& $\{y\}^{\sspace}$& $\{y\}^{\sspace}$ & $\{y\}^{\sspace}$& $\{y\}^{\sspace}$ & $O(1)$\\ 
    Model-irrelevance  & $\prob$, $\rew$  & $\prob$, $\rew$ &$\prob$, $\rew$ &$\prob$, $\rew$ \\
    $Q^{\pi}$-irrelevance & $Q^{\pi}$ &$Q^{\pi}$ &$Q^{\pi}$ & $Q^{\pi}$\\
    $Q^*$-irrelevance  & $Q^*$ &$Q^*$ &$Q^*$ & $Q^*$ \\
    $a^*$-irrelevance  & $Q^*$ & $Q^*$&$Q^*$ & $Q^*$ \\
    Approx. abstraction & - &-  &- & -\\
    $e^{\sim}$& $Q^*$   &  $Q^*$  & $Q^*$ &$Q^*$ & $O(|\aspace||\sspace|^3)$   \\
    $d^{\sim}$ & $Q^*$  &$Q^*$  & $Q^*$ &$Q^*$ & $O\big(|\cA||\cS|^5\log |\cS|\frac{\ln\delta}{\ln \gamma}\big)$ \\
    $e^{\sim_{\pi}}$ &  $Q^{\pi}$& $Q^{\pi}$& $Q^{\pi}$ & $Q^{\pi}$ & $O(|\sspace|^3)$  \\
    $d^{\sim_{\pi}}$ & $Q^{\pi}$  &$Q^{\pi}$ & $Q^{\pi}$ & $Q^{\pi}$ & $O\big(|\cS|^5\log |\cS|\frac{\ln\delta}{\ln \gamma}\big)$\\
    $e^{\sim_{lax}}$ &$V^*$   &$V^*$  & $V^*$ &$V^*$ & $O(|\aspace|^2|\sspace|^3)$ \\
    $d^{\sim_{lax}}$& $V^*$  &$V^*$  & $V^*$ &$V^*$  &$O\big(|\cA|^2|\cS|^5\log |\cS|\frac{\ln\delta}{\ln \gamma}\big)$\\
    $d_{\Delta^*}$ & $Q^*$  &  $Q^*$ &$Q^*$&$Q^*$ &$O\big(|\sspace|^2|\aspace|\frac{\log(\rew_{\text{max}}^{-1}\delta(1-\gamma))}{\log(\gamma)}\big)$ \\
    $d_{\Delta_{\pi}}$ &$Q^{\pi}$&$Q^{\pi}$&$Q^{\pi}$&$Q^{\pi}$ &$O\big(|\sspace|^2|\aspace|\frac{\log(\rew_{\text{max}}^{-1}\delta(1-\gamma))}{\log(\gamma)}\big)$\\ 
    $d_{\Delta_{\forall}}$& $Q^{\pi}$, $\forall \pi \in \Pi$ & $Q^{\pi}$, $\forall \pi \in \Pi$ &$Q^{\pi}$, $\forall \pi \in \Pi$&$Q^{\pi}$, $\forall \pi \in \Pi$  & NP-hard? \citep{bellemare19geometric}   \\

    \bottomrule
  \end{tabularx}
    \vspace{-1em}
      \label{fig:largertable}
\end{table*}
We now formally present in \autoref{the:topologyships} the topological relationships between the different metrics. This
hierarchy is important for generalization purposes as it
provides a comparison between the shapes of different
neighbourhoods which serve as a basis for RL algorithms on continuous state spaces.
\begin{theorem}
\label{the:topologyships}
  The relationships between the topologies induced by the metrics in \autoref{fig:largertable} are given by the following diagram.  We denote by $d_1\rightarrow d_2$ when $\mathcal{T}_{d_1} \subset  \mathcal{T}_{d_2}$, that is, when $\cT_{d_1}$ is coarser than $\cT_{d_2}$. Here $d$ denotes any arbitrary metric.
  \begin{equation*}
    \begin{tikzcd}
      e^{\sim_{lax}} \arrow[d] &d^{\sim_{lax}} \arrow[l]  \arrow[d] & d_{\widetilde{\text{AVF}}(n)} \arrow{dr} & d  \arrow[r]  &e^{\rI} \\
      e^{\sim}  & d^{\sim}\arrow[l]& d_{\Delta^*}\arrow[l]\arrow[r] & d_{\Delta_{\forall}} &  d_{\Delta_{\pi}}\arrow[l] \arrow[d]&&& & & \\
      d& e^{\rT} \arrow[l] & & e^{\sim_{\pi}}& d^{\sim_{\pi}} \arrow[l]& 
    \end{tikzcd}
    \label{diag:metricRel}
  \end{equation*}
\end{theorem}
\begin{proof}
All proofs can be found in the appendix. The relation $d^{\sim_{lax}} \rightarrow d^{\sim}$ was shown by \citet{taylor2009bounding} but not expressed in topological terms.
\end{proof}
We summarize in \autoref{fig:largertable} our continuity results mentioned throughout this section and supplement them with the continuity of the lax-bisimulation metric proven in \cite{taylor2009bounding}. To avoid over-cluttering the table, we only specify the strongest form of functional continuity according to \autoref{thm:contRel}.
As an additional key differentiator, we also note the complexity of computing these metrics from a full model of the environment, which gives some indication about the difficulty of performing state abstraction. Proofs are provided in the appendix.

From a computational point of view, all continuous metrics can be approximated using deep learning techniques which makes them even more attractive to build representations. Atari 2600 experiments by \citet{castro2020scalable} show that $\pi$-bisimulation metrics do perform well in larger domains. This is also supported by \citep{zhang2020learning} who use an encoder architecture to learn a representation that respects the bisimulation metric.

\section{Empirical evaluation}
\label{sec:empirical}
\begin{figure*}[h!]
  \centering
  \includegraphics[width=0.3\textwidth]{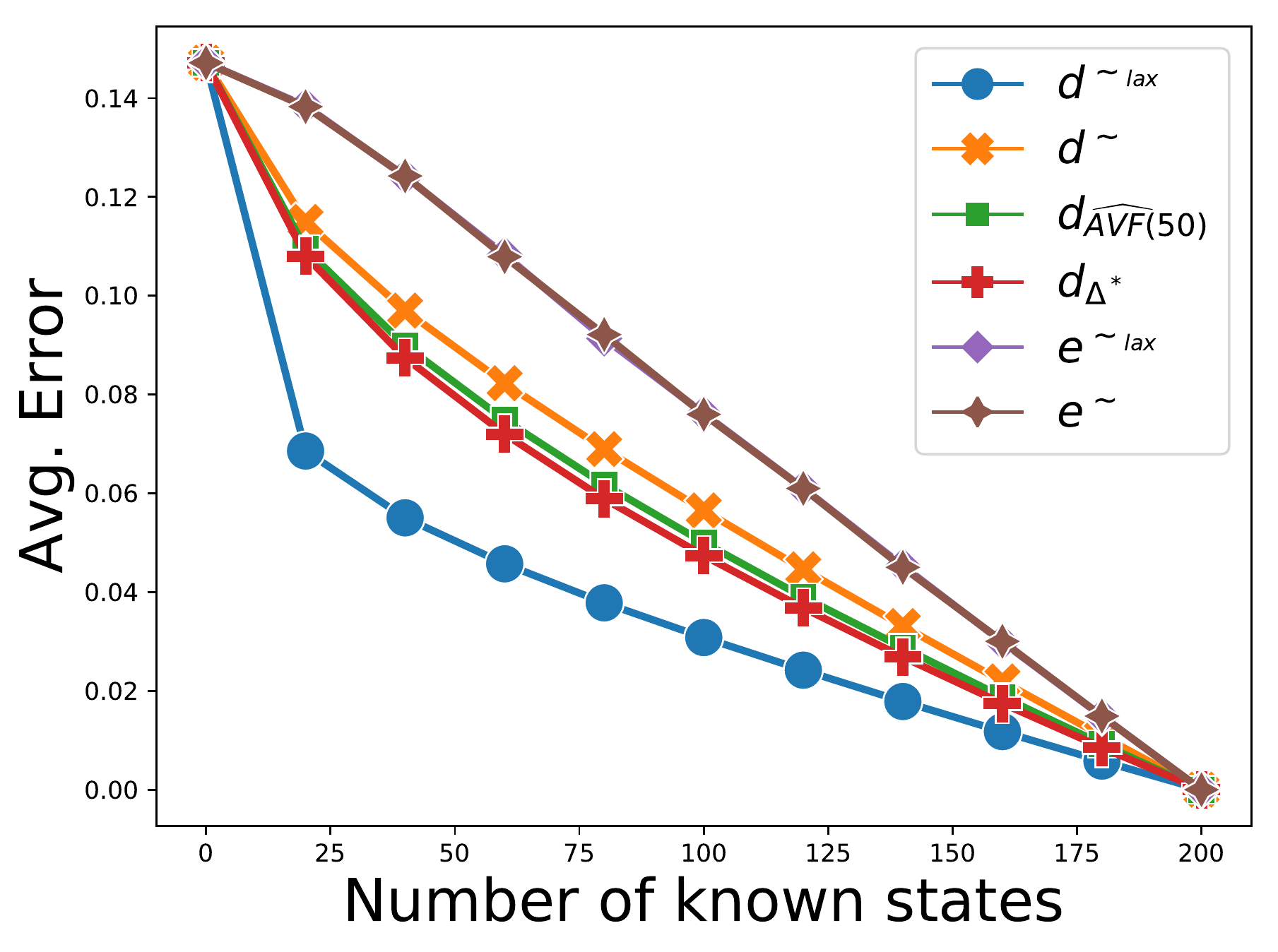}
    \includegraphics[width=0.3\textwidth]{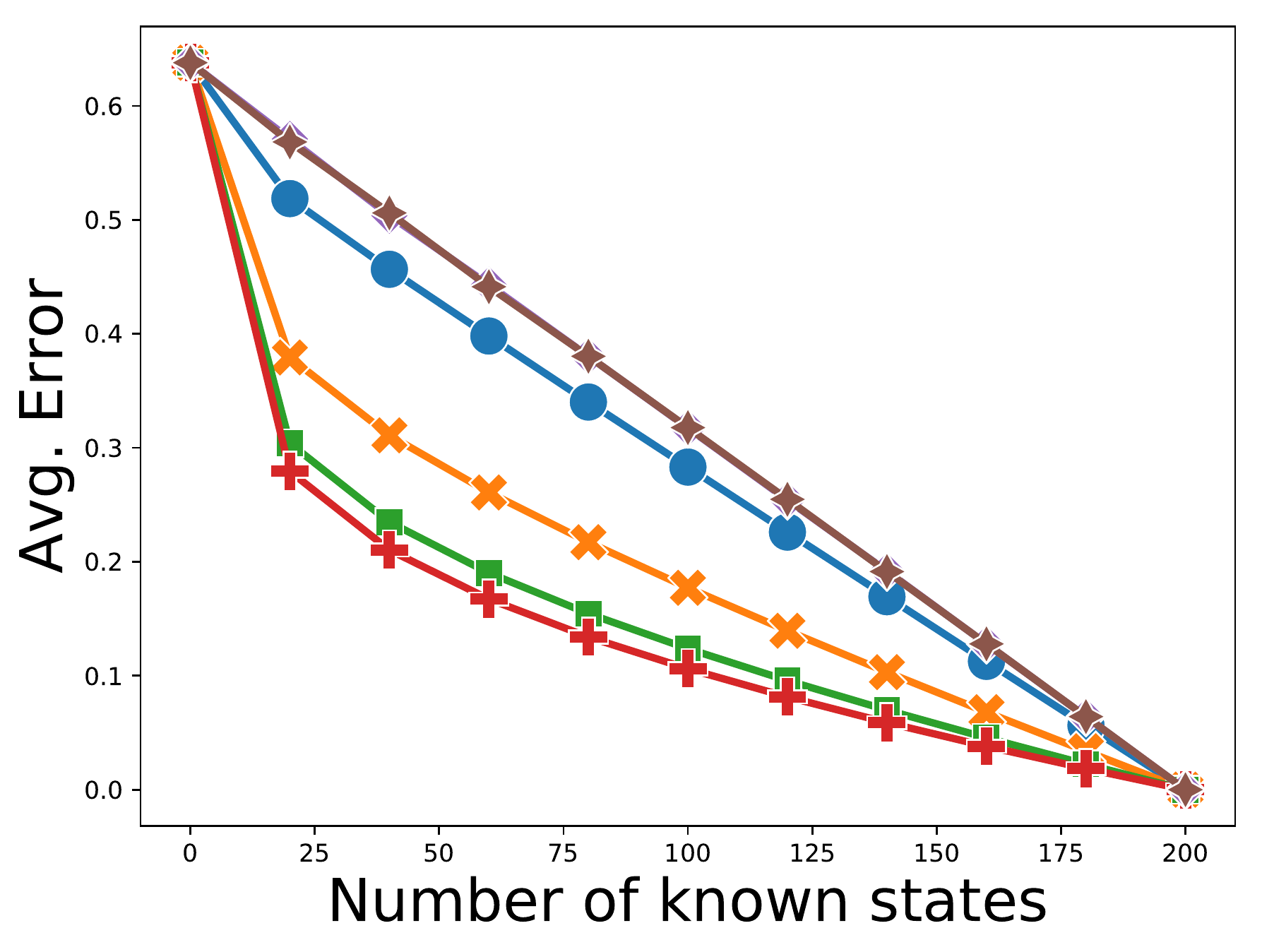}
  \includegraphics[width=0.3\textwidth]{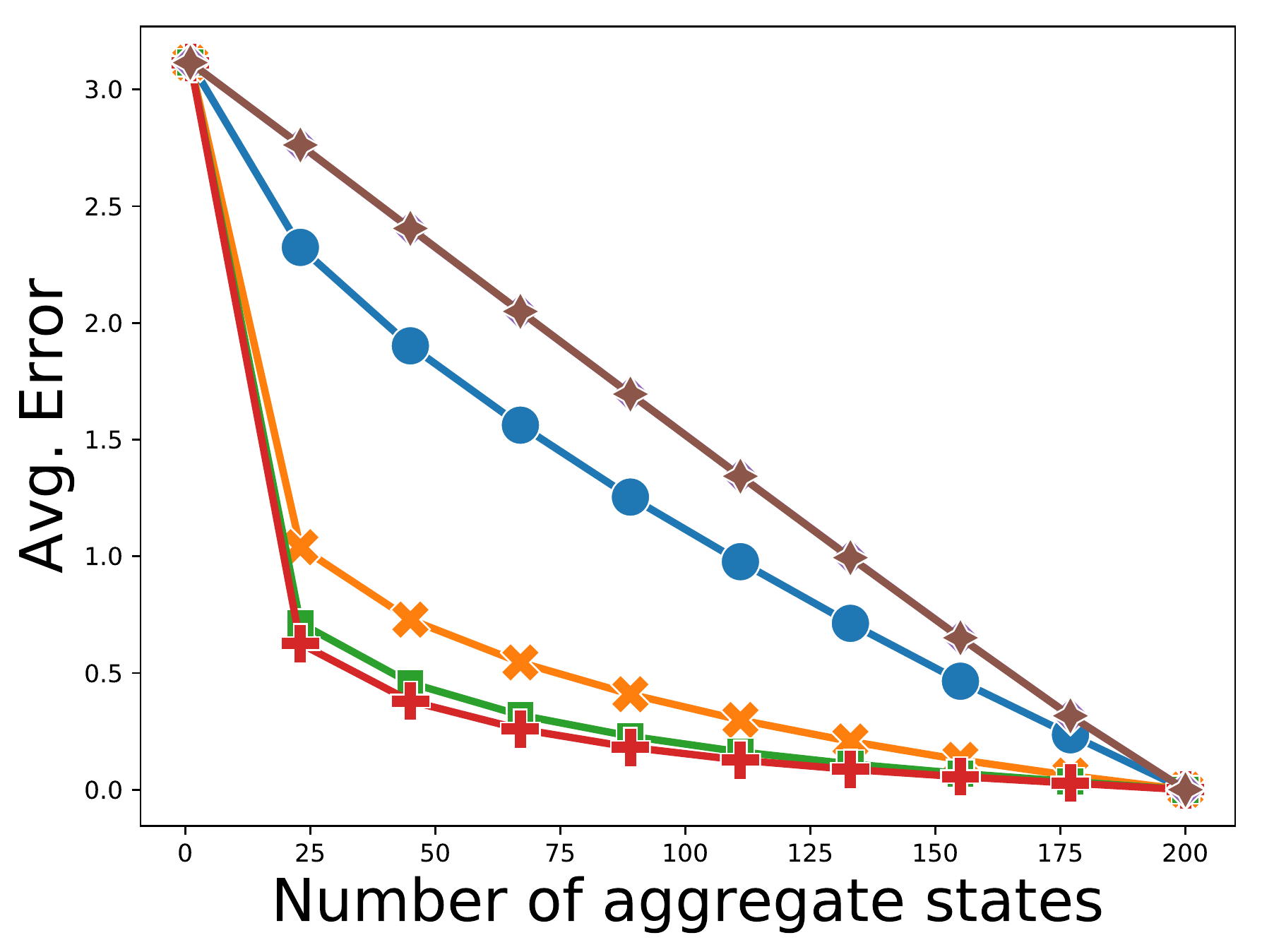}
  \caption{Errors when approximating the optimal value function (left) and optimal Q-function (center) via nearest-neighbours and errors when performing value iteration on aggregated states (right). Curves for $e^{\sim}$ and $e^{\sim_{lax}}$ are covering each other on all of the plots. Averaged over 100 Garnet MDPs with 200 states and 5 actions, with 50 independent runs for each (to account for subsampling differences). Confidence intervals were very tiny due to the large number of runs so were not included.} 
  \label{fig:subsamplingAggregation}
\end{figure*}
We now conduct an empirical evaluation to quantify the magnitude of the effects studied in the previous sections.
Specifically, we are interested in how approximations derived from different metrics impact the performance of basic reinforcement learning procedures. We consider two kinds of approximations: state aggregation and nearest neighbour, which we combine with six representative metrics: $e^{\sim}$,  $e^{\sim_{lax}}$, $d^{\sim}$, $d^{\sim_{lax}}$, $d_{\Delta^*}$, and $d_{\widetilde{\text{AVF}}(50)}$.

We conduct our experiments on Garnet MDPs, which are a class of randomly generated MDPs \citep{archibald95generation,piot14difference}. Specifically, a Garnet MDP $Garnet(n_{\cS}, n_{\cA})$ is parameterized by two values: the number of states $n_{\cS}$ and the number of actions $n_{\cA}$, and is generated as follows: {\bf 1.} The branching factor $b_{s,a}$ of each transition $\cP_s^a$ is sampled uniformly from $[1:n_{\cS}]$. {\bf 2.} $b_{s, a}$ states are picked uniformly randomly from $\cS$ and assigned a random value in $[0, 1]$; these values are then normalized to produce a proper distribution $\cP_s^a$. {\bf 3.} Each $\cR_s^a$ is sampled uniformly in $[0, 1]$. The use of Garnet MDPs grants us a less-biased comparison of the different metrics than if we were to pick a few specific MDPs. Nonetheless, we do provide extra experiments on a set of GridWorld tasks in the appendix.
The code used to produce all these experiments is open-sourced \footnote{Code available at \url{https://github.com/google-research/google-research/tree/master/rl_metrics_aaai2021}}.

\subsection{Generalizing the value function $V^*$}
We begin by studying the approximation error that arises when extrapolating the optimal value function $V^*$ from a subset of states.
Specifically, given a subsampling fraction $f \in [0, 1]$, we sample $\lceil |\sspace| \times f \rceil$ states and call this set $\kappa$. For each unknown state $s\in\sspace\setminus \kappa$, we find its nearest known neighbour according to metric $d$: $NN(s) = \arg\min_{t\in \kappa}d(s, t)$.
We then define the optimal value function as $\hat{V}^*(s) = V^*(NN(s))$, and report the approximation error in \autoref{fig:subsamplingAggregation} (left). This experiment gives us insights into how amenable the different metrics are for transferring value estimates across states; effectively, their generalization capabilities.

According to \autoref{the:topologyships}, the two discrete metrics  $e^{\sim}$ and $e^{\sim_{lax}}$ induce finer topologies than their four continuous counterparts. Most of the states being isolated from each other in these two representations, $e^{\sim}$ and $e^{\sim_{lax}}$ perform poorly.
The three continuous metrics $d^{\sim}$, $d^{\sim_{lax}}$ and $d_{\Delta^*}$ all guarantee Lipschitz continuity of $V^*$ while $d_{\widetilde{\text{AVF}}(50)}$ is approximately $V^*$ Lipschitz continuous. However, $d^{\sim_{lax}}$ (resp. $d_{\Delta^*}$) produce coarser (resp. approximately coarser) topologies than $d^{\sim}$ (resp. $d_{\widetilde{\text{AVF}}(50)}$) (see Theorem \ref{the:topologyships}). This is reflected in their better generalization error compared to the latter two metrics. Additionally, the lax bisimulation metric $d^{\sim_{lax}}$ outperforms $d_{\Delta^*}$ substantially, which can be explained by noting that $d^{\sim_{lax}}$ measures distances between two states under independent action choices, contrary to all other metrics.

\subsection{Generalizing the Q-function $Q^*$}
We now illustrate the continuity (or absence thereof) of $Q^*$ with respect to the different metrics.
In \autoref{fig:subsamplingAggregation} (center), we perform a similar experiment as the previous one, still using a 1-nearest neighbour scheme but now extrapolating $Q^*$. 

As expected, we find that metrics that do not support $Q^*$ continuity, including $d^{\sim_{lax}}$, cannot generalize from a subset of states, and their average error decreases linearly. In contrast, the three other metrics are able to generalize. Naturally, $d_{\Delta^*}$, which aggregates states based on $Q^*$, performs particularly well. However, we note that $d_{\Delta_{\forall}}$ also outperforms the bisimulation metric $d^{\sim}$, highlighting the latter's conservativeness, which tends to separate states more. By our earlier argument regarding $d^{\sim_{lax}}$, this suggests there may be a class of functions, not represented in Table 2, which is continuous under $d^{\sim}$ but not $d_{\widetilde{\text{AVF}}(50)}$.

\subsection{Approximate value iteration}
As a final experiment, we perform approximate value iteration using a state aggregation $\phi$ derived from one of the metrics. For each metric, we perform 10 different aggregations using a $k$-median algorithm, ranging from one aggregate state to 200 aggregate states.
For a given aggregate state $c$, let $Q(c,a)$ stand for its associated Q-value. The approximation value iteration update is
\[ \hat Q_k(c, a) \gets \frac{1}{|c|}\sum_{s | \phi(s) = c}\left[ \rew^a_s + \gamma \rE_{s'\sim\cP_s^a} \max_{a\in\aspace}\hat Q_k(\phi(s')) \right] \]
We can then measure the error induced by our aggregation via $\max_{a\in\aspace}\frac{1}{|\sspace|}\sum_{s\in\sspace}|Q^*(s, a) - \hat Q_k(\phi(s), a)|$, which we display in the rightmost panel of \autoref{fig:subsamplingAggregation}. 

As in our second experiment, the metrics that do not support $Q^*$-continuity well fail to give good abstractions for approximate value iteration.
As for $e^{\sim}$, the topology induced by this metric is too fine (\autoref{the:topologyships}) leading to poor generalisation results.
The performance of $d_{\Delta^*}$ is consistent with \autoref{the:topologyships}, which states that it induces the coarsest topology. However, although it is known that $Q^*$-continuity is sufficient for approximate value iteration \citep{li2006towards}, it is somewhat surprising that it outperforms $d_{\widetilde{\text{AVF}}(50)}$,
since $d_{\widetilde{\text{AVF}}(50)}$ is an approximant of $d_{\Delta_{\forall}}$ that is designed to provide continuity with respect to all policies, so it may be expected to yield better approximations at intermediate iterations. Despite this, $d_{\Delta_{\forall}}$ still serves as an interesting and tractable surrogate metric to $d_{\Delta^*}$.

\section{Discussion}
\label{sec:conclusion}

Behavioral metrics are important both to evaluate the goodness of a given state representation and to learn such a representation.
We saw that approximate abstractions and equivalence relations are insufficient for continuous-state RL problems, because they do not support the continuity of common RL functions or induce very fine representations on the state space leading to poor generalization.

Continuous behavioural metrics go one step further by considering the structure of the MDP in their construction and inducing coarser topologies than their discrete counterparts; however, within that class we still find that not all metrics are equally useful.
The original bisimulation metric of \citet{ferns2004metrics}, for example, is too conservative and has a rather fine topology.
This is confirmed by our experiments in \autoref{fig:subsamplingAggregation}, where it performs poorly overall. The lax bisimulation metric guarantees the continuity of $V^*$ which makes it suitable for transferring optimal values between states but fails to preserve continuity of $Q^*$.
Together with our analysis, the $d_{\Delta^*}$ and $d_{\Delta_{\forall}}$ metrics seem interesting candidates when generalising within a neighbourhood.

$d_{\Delta_{\forall}}$ is useful when we do not know the value improvement path the algorithm will be following \citep{dabney2020value}. Despite being approximated from a finite number of policies, the performance of $d_{\widetilde{\text{AVF}}(n)}$, reflects the fact that it respects, in some sense, the entire space of policies that are spanned by policy iteration and makes it useful in practice. One advantage of this metric is that it is built from value functions, which are defined on a per-state basis; this makes it amenable to online approximations. In contrast, bisimulation metrics are only defined for pairs of states, which makes it difficult to approximate in an online fashion, specifically due to the difficulty of estimating the Wasserstein metric on every update.

Finally, continuing our analysis on partially observable systems is an interesting area for future work. Although \citet{castro09notions} proposed various equivalence relations for partially observable systems, there has been little work in defining proper metrics for these systems.

\section{Broader Impact}
This work lies in the realm of ``foundational RL'' in that it contributes to the fundamental understanding and development of reinforcement learning algorithms and theory. As such, despite us agreeing in the importance of this discussion, our work is quite far removed from ethical issues and potential societal consequences.

\section*{Acknowledgements}
The authors would like to thank Sheheryar Zaidi, Adam Foster and Abe Ng for insightful discussions about topology and functional analysis, Carles Gelada, John D. Martin, Dibya Ghosh, Ahmed Touati, Rishabh Agarwal, Marlos Machado and the whole Google Brain team in Montreal for helpful discussions, and Robert Dadashi for a conversation about polytopes.  We also thank Lihong Li and the anonymous reviewers for useful feedback on this paper.

We would also like to thank
the Python community~\citep{van1995python,oliphant2007python} and in particular
{NumPy}~\citep{oliphant2006guide,walt2011numpy, harris2020array}, Tensorflow \citep{abadi2016tensorflow},
{SciPy}~\citep{jones2001scipy},
{Matplotlib}~\citep{hunter2007matplotlib} and {\em
\href{https://github.com/google/gin-config}{Gin-Config}}.

\bibliography{metrics}
\clearpage
{
\appendix
\onecolumn
{\Large \bf Appendix}
\pdfoutput=1
\section{Proofs of the section Continuity relationships}
\subsection{Continuity relationships}
\begin{table}
  \centering
  \begin{tabular}{|c|c|c|}
    \textbf{Function(s)} & \textbf{Domain} & \textbf{Range} \\
    \hline
    $\prob$ & $\sspace\times\aspace$ & $\Sigma\rightarrow [0, 1]$ \\
    $\rew$ & $\sspace\times\aspace$ & $[0, R_{\text{max}}]\subset\rR$ \\
    $V^{\pi}, V^*$ & $\sspace$ & $[0, V_{\text{max}}]\subset\rR$ \\
    $Q^{\pi}, Q^*$ & $\sspace\times\aspace$ & $[0, V_{\text{max}}]\subset\rR$ \\
    $\pi$ & $\sspace$ & $\Delta(\aspace)$ \\
    \hline
  \end{tabular}
  \caption{RL functions with their respective domains and ranges.}
  \label{tbl:contRel}
\end{table}
We begin by proving the first main theorem in the paper, Theorem~1. We report in \autoref{tbl:contRel} the domains and ranges of the different RL functions mentioned in Theorem 1 that will be used throughout the proof. Before proving this result, we introduce the following necessary lemma.

 \begin{lemmaap}\label{the:Qcont}
 Choosing the discrete topology on the finite space $\aspace$ and assuming the product metric $d_{\sspace \times \aspace}=d_{\sspace} + d_{\aspace}$, the function $Q^{\pi}:(\sspace \times \aspace, d_{\sspace \times \aspace}) \rightarrow \rR$ is continuous if and only if
 \begin{enumerate}
   \item The function $Q^{\pi}:(\sspace \times \{a\}, d_{\sspace \times \aspace})\rightarrow \rR$ is continuous for all $a \in \aspace$.
   \item The function $Q^{\pi}(\cdot, a):(\sspace, d_{\sspace})\rightarrow \rR$ is continuous for all $a \in \aspace$.
 \end{enumerate}
 \end{lemmaap}
\begin{proof}
To understand better the notion of continuity on the space $\aspace$ endowed with the discrete topology, we refer the reader to ~\autoref{the:DiscreteMetricCont}.

We begin with the first equivalence.

$(\implies):$ For LC, UC, LLC and ULC, this result follows from the fact that the function on $\sspace \times \{a\}$ is a restriction of the function on $\sspace \times \aspace$. For instance in the case of LC, suppose $Q^{\pi}$ is LC on $\sspace \times \aspace$. Let $\epsilon>0$. Then, for all $(s, a), (s',a') \in \sspace \times \aspace$ there exists $\delta>0$ such that, $d_{\sspace \times\aspace}((s, a), (s', a'))\leq \delta \implies |Q^{\pi}(s, a)-Q^{\pi}(s', a')|\leq \epsilon$. In particular, this is true for $a=a'$ so $Q^{\pi}$ is LC on $\sspace \times \{a\}$ for all $a \in \aspace$.

$(\impliedby):$ $Q^{\pi}$ is LC on $\sspace \times \{a\}$ for all $a \in \aspace$. So for all $a \in \aspace$, the limit of $Q(s_n, a)$ as the sequence $s_n\in \sspace$ converges to $s \in \sspace$ exists and is equal to $Q(s, a)$, that is  $s_n \rightarrow s \implies Q(s_n, a) \rightarrow Q(s, a)$. Moreover, $(s_n, a_n) \rightarrow (s, a_0)$ implies that $a_n= a_0$ for all $n$ big enough because $\aspace$ has the discrete metric. So for $n>N\in \mathbb{N}, Q(s_n, a_n)=Q(s_n, a_0)\rightarrow Q(s, a_0)$. Hence, $Q^{\pi}:(\sspace \times \{a\}, d_{\sspace \times \aspace}) \rightarrow \rR \text{ is LC for all a} \in \aspace \implies Q^{\pi}:(\sspace \times \aspace, d_{\sspace \times \aspace}) \rightarrow \rR$ is LC.

  Now, in the UC case, $Q^{\pi}$ is UC on $\sspace \times \{a\}$ for all $a \in \aspace$, so we have: $\forall a \in \aspace, \forall \epsilon >0, \exists \delta_{\epsilon, a}>0,$ such that for all $s, s' \in \sspace, d_{\sspace \times \aspace}((s, a), (s', a))\leq \delta_{\epsilon, a} \implies |Q^{\pi}(s, a)-Q^{\pi}(s', a)|<\epsilon.$   The space $\aspace$ being finite, $\min\limits_{a \in \aspace}\delta_{\epsilon, a}$ exists and is positive. Hence,  $\forall \epsilon >0,$ there exists $\delta_1=\min{ \{\min\limits_{a \in \aspace}\delta_{\epsilon, a}, 1/2\}} >0,$ such that for all $(s, a), (s', a') \in \sspace \times \aspace$, if $d_{\sspace \times \aspace}((s, a), (s', a'))\leq \delta_1 < 1/2$, then $a = a'$ since $d_{\aspace}(a, a')$ can only take values $0$ or $1$. Applying UC of $Q^{\pi}( \cdot, a)$, we get $d_{\sspace \times \aspace}((s, a), (s', a'))\leq \delta_1 \leq \delta_{\epsilon, a} \implies |Q^{\pi}(s, a)-Q^{\pi}(s', a')|\leq\epsilon.$ We can conclude that $Q^{\pi}$ is UC on $\sspace \times \aspace$.\\ \\
In the ULC case, $Q^{\pi}$ is ULC on $\sspace \times \{a\}$ for all $a \in \aspace$, so we have: $\forall a \in \aspace, \exists L_{a}>0,$ such that for all $s, s' \in \sspace, |Q^{\pi}(s, a)-Q^{\pi}(s', a)| \leq L_{a} d_{\sspace \times \aspace}((s, a), (s', a)).$ So, as $Q^{\pi}$ is bounded by $V_{\max}$, there exists $L=\max\{\max\limits_{a \in \aspace}L_{a}, V_{\max}\} \geq 0,$ such that for all $(s, a), (s', a') \in \sspace \times \aspace, |Q^{\pi}(s, a)-Q^{\pi}(s', a')| \leq L d_{\sspace \times \aspace}((s, a), (s', a'))=L(d_{\sspace}(s, s') + d_{\aspace}(a, a')).$

In the LLC case, $Q^{\pi}$ is LLC on $\sspace \times \{a\}$ for all $a \in \aspace$. So for all $(s, a) \in \sspace \times \aspace$, there exists a neighbourhood $U$ of $\sspace \times \aspace$ induced by $d_{\sspace \times \aspace}$ such that $Q^{\pi}$ restricted to $U$ is ULC. We conclude by the same argument as in the ULC case above.

The second equivalence is true because, for any $a \in \aspace$, the relabelling map $s \mapsto (s, a)$ is an isometry of metric spaces $(\sspace \times \{a\}, d_{\sspace \times \aspace})$ and $(\sspace, d_{\sspace})$. Isometric metric spaces are "equivalent" and hence have the same properties.
\end{proof}

\contRel*

\begin{equation}
  \begin{tikzcd}
    & & Q^{\pi}  \arrow[rr, "\pi \text{-cont}"]& & V^{\pi} \\
    \prob + \rew \arrow[drr] \arrow[urr, "\pi \text{-cont } "]&  & & &  \\
    & & Q^* \arrow[rr]& & V^{*}
  \end{tikzcd}
    \label{diag:contRel}
\end{equation}

The proof itself will be made up of a series of lemmas for each of the arrows (or lack thereof) in the diagram.

\begin{lemmaap}
If $Q^*:(\sspace \times \aspace, d_{\sspace \times \aspace}) \rightarrow \rR \text{ is continuous}$, then $V^*:\sspace  \rightarrow \rR$ is continuous.
\end{lemmaap}
By definition of the optimal value function,
\begin{equation*}
 V^*(s)=\max_{a \in \aspace}Q^*(s, a).
 \end{equation*}
$\aspace$ being a finite set of discrete actions and the \textit{max} function being a non-expanding (that is, 1-Lipschitz) map, it results that if $Q^*$ is LC (resp. UC, resp. ULC, resp. LLC) then $V^{*}$ is LC (resp. UC, resp. ULC, resp. LLC).\\

In more details, let's assume $Q^*$ is LC. Let $a \in \aspace$ and $\epsilon > 0$. By definition of LC of $Q^*$, there exists $\delta_{s, \epsilon}$ such that for all $s' \in \sspace$, $d_{\sspace}(s, s')\leq \delta_{s, \epsilon} \implies |Q^*(s, a)- Q^*(s', a)|\leq \epsilon$. 
\begin{align*}
  |V^*(s_1) - V^*(s_2)|& = |\max_{a \in \aspace}Q^*(s_1, a) - \max_{a \in \aspace}Q^*(s_2, a)| \\
  &\leq \max_{a \in \aspace} |Q^*(s_1, a)- Q^*(s_2, a)| \text{ as max is a non expansion.}\\
  &\leq \max_{a \in \aspace} \epsilon \text{ as $Q^*$ is LC.}\\
  &\leq \epsilon 
  \end{align*}
Hence, $V^*$ is LC. The proof for the UC case is similar.\\
Now, in the ULC case: Let $s_1, s_2 \in \sspace$,
\begin{align*}
|V^*(s_1) - V^*(s_2)|& = |\max_{a \in \aspace}Q^*(s_1, a) - \max_{a \in \aspace}Q^*(s_2, a)| \\
&\leq \max_{a \in \aspace} |Q^*(s_1, a)- Q^*(s_2, a)| \text{ as max is a non expansion.}\\
&\leq \max_{a \in \aspace} L_a d(s_1, s_2) \text{ as $Q^*$ is ULC.}\\
& \leq L d(s_1, s_2).
\end{align*}
We can thus conclude that $V^*$ is also ULC. The LLC case is similar to the ULC proof.\\ 
The reverse implication is not true as shows the following counter-example. Suppose $\sspace=\rR$, $\aspace = \{1, 2\}$ and $Q(s, 1) = 1$ for all s. And suppose Q(s, 2) is some discontinuous function that is always less than 1. Let's for instance choose: 
\begin{align}
Q^{*}(s, 2) & = \left\{
    \begin{array}{ll}{}
        0 & \mbox{if } s\leq s_0, s_0 \in \rR  \\
        0.5 & \mbox{if } s>s_0.
    \end{array}
\right.
\end{align}

Then $V^*(s) = Q^*(s, 1)$ which is continuous but $Q^*(s, a)$ is not continuous at $s_0$ for $a = 2$.
\begin{lemmaap}
\label{lemmac}
If $Q^{\pi}:(\sspace \times \aspace, d_{\sspace \times \aspace}) \rightarrow \rR$ and $\pi: \sspace \rightarrow \Delta(\aspace)$ are continuous, then $V^{\pi}$ is continuous.
\end{lemmaap}
The value function $V^{\pi}$ is defined as follows: 
\begin{equation*}
V^{\pi}(s) = \sum_{a \in \aspace}\pi(a|s)Q^{\pi}(s, a).
\end{equation*}
In the LC case: let's assume $Q^{\pi}$ is LC. Let $a \in \aspace$ and $\epsilon > 0$. By definition of LC of $Q^{\pi}$, there exists $\delta_{s, \epsilon}$ such that for all $s' \in \sspace$, $d_{\sspace}(s, s')\leq \delta_{s, \epsilon} \implies |Q^{\pi}(s, a)- Q^{\pi}(s', a)|\leq \epsilon$. We also assume the policy is LC, that is, there exists $\delta_{s, \epsilon}^{'}$ such that for all $s' \in \sspace$, $d_{\sspace}(s, s')\leq \delta^{'}_{s, \epsilon} \implies W_{d_{\aspace}}(\pi(\cdot|s)-\pi(\cdot|s'))|\leq \epsilon$. For all $s'\in \sspace$ such that $d_{\sspace}(s, s')\leq \min(\delta^{'}_{s, \epsilon}, \delta_{s, \epsilon})$, we have
\begin{align*}
  |V^{\pi}(s)- V^{\pi}(s')| &= |\mathbb{E}_{\substack{a \sim \pi(\cdot|s)}}Q^{\pi}(s, a) - \mathbb{E}_{\substack{a \sim \pi(\cdot|s')}}Q^{\pi}(s', a)|\\
  &= |\mathbb{E}_{\substack{a \sim \pi(\cdot|s)}}Q^{\pi}(s, a) -\mathbb{E}_{\substack{a \sim \pi(\cdot|s')}}Q^{\pi}(s, a) + \mathbb{E}_{\substack{a \sim \pi(\cdot|s')}}Q^{\pi}(s, a) - \mathbb{E}_{\substack{a \sim \pi(\cdot|s')}}Q^{\pi}(s', a)|\\
  & \leq|\mathbb{E}_{\substack{a \sim \pi(\cdot|s)}}Q^{\pi}(s, a) -\mathbb{E}_{\substack{a \sim \pi(\cdot|s')}}Q^{\pi}(s, a)| + |\mathbb{E}_{\substack{a \sim \pi(\cdot|s')}}Q^{\pi}(s, a) - \mathbb{E}_{\substack{a \sim \pi(\cdot|s')}}Q^{\pi}(s', a)|\\
  & \leq |V_{\text{max}} W_{d_{\aspace}}(\pi(\cdot|s)-\pi(\cdot|s'))| + \mathbb{E}_{\substack{a \sim \pi(\cdot|s')}}|Q^{\pi}(s, a) - Q^{\pi}(s', a)| \text{ by definition of the Wasserstein}\\
  & \leq (V_{\text{max}}+1) \epsilon
\end{align*}
So $V^{\pi}$ is LC. The proof is similar in the UC case.

In the ULC case: let $s, s' \in \sspace$,

\begin{align*}
  |V^{\pi}(s)- V^{\pi}(s')| &= |\mathbb{E}_{\substack{a \sim \pi(\cdot|s)}}Q^{\pi}(s, a) - \mathbb{E}_{\substack{a \sim \pi(\cdot|s')}}Q^{\pi}(s', a)|\\
  &= |\mathbb{E}_{\substack{a \sim \pi(\cdot|s)}}Q^{\pi}(s, a) -\mathbb{E}_{\substack{a \sim \pi(\cdot|s')}}Q^{\pi}(s, a) + \mathbb{E}_{\substack{a \sim \pi(\cdot|s')}}Q^{\pi}(s, a) - \mathbb{E}_{\substack{a \sim \pi(\cdot|s')}}Q^{\pi}(s', a)|\\
  & \leq|\mathbb{E}_{\substack{a \sim \pi(\cdot|s)}}Q^{\pi}(s, a) -\mathbb{E}_{\substack{a \sim \pi(\cdot|s')}}Q^{\pi}(s, a)| + |\mathbb{E}_{\substack{a \sim \pi(\cdot|s')}}Q^{\pi}(s, a) - \mathbb{E}_{\substack{a \sim \pi(\cdot|s')}}Q^{\pi}(s', a)|\\
  & \leq |V_{\text{max}} W_{d_{\aspace}}(\pi(\cdot|s)-\pi(\cdot|s'))| + \mathbb{E}_{\substack{a \sim \pi(\cdot|s')}}|Q^{\pi}(s, a) - Q^{\pi}(s', a)| \text{ by definition of the Wasserstein}\\
  & \leq V_{\text{max}} L_{\pi} d(s, s')+ \mathbb{E}_{\substack{a \sim \pi(\cdot|s')}}\max_{a \in \aspace}L_a d(s, s') \text{ as the policy and Q-functions are ULC}\\
  & \leq (V_{\text{max}} L_{\pi} + \max_{a \in \aspace}L_a)d(s, s')\
\end{align*}
We can thus conclude that $V^{\pi}$ is also ULC with Lipschitz constant $(V_{\text{max}} L_{\pi} + \max_{a \in \aspace}L_a)$. The same reasoning applies in the LLC case. \\ \\
We emphasize that the continuity assumption of $\pi$ is important, as this example shows. Let's assume $\sspace = \rR$ and $\aspace=\{1, 2\}$. Imagine we have the following discontinuous policy:
\begin{align}
\pi(a|s) & = \left\{
    \begin{array}{ll}{}
        \delta_0 & \mbox{if } s < 0 \\
        \delta_1 & \mbox{if } s \geqslant 0.
    \end{array}
\right.
\end{align}
where:
\begin{align}
\delta_{x_0}(A)& = \left\{
    \begin{array}{ll}{}
        1& \mbox{if } x_0 \in A \\
        0& \text{ else}.
    \end{array}
\right.
\end{align}
and the following value function:
\begin{align}
Q^{\pi}(s, a) & = \left\{
    \begin{array}{ll}{}
        s & \mbox{if } a = 0 \\
        s + 1 & \mbox{if } a = 1.
    \end{array}
\right.
\end{align}
Then, $V^{\pi}$ is discontinuous at $0$.
\begin{align}
V^{\pi}(s) & = \left\{
    \begin{array}{ll}{}
        s & \mbox{if } s < 0 \\
        s + 1 & \mbox{if } \mbox{if } s \geqslant 0.
    \end{array}
\right.
\end{align}
It is clear that continuity of $V^{\pi}$ does not imply continuity of $Q^{\pi}$: take the deterministic constant constant policy $\pi(a|s)={1}$ so that $V^{\pi}(s) = Q^{\pi}(s, 1).$ As previously, we can have any discontinuous function for $ Q^{\pi}(s, 1)$.
\begin{corollary}
\label{lemma1}
In the deterministic case, if $\pi: \sspace \rightarrow \aspace$ is continuous, $\sspace$ is connected\footnote{A connected space is topological space that cannot be represented as the union of two or more disjoint non-empty open subsets} and $\aspace$ is discrete, then $\pi$ is globally constant.
\end{corollary}
\begin{proof}
$\pi$ is continuous at $s$ iff for all $\epsilon >0$, there exists $\delta_{\epsilon, s} >0$ such that for all $s' \in \sspace$, $d(s, s')\leq \delta_{\epsilon, s}$ implies $d_{\aspace}(\pi(s)-\pi(s'))\leq \epsilon$.
In particular, choosing $\epsilon=\frac{1}{2}$ implies that $\pi(s)=\pi(s').$ $\pi$ it thus locally constant. Supposing $\sspace$ is connected implies $\pi$ is globally constant.
\end{proof}
While our proof above is valid for stochastic policies, we note that the proof of \autoref{lemmac} in the ULC case with deterministic policies is provided by \citet{rachelson2010locality}:
\begin{corollary}[\citeauthor{rachelson2010locality}, \citeyear{rachelson2010locality}]
If $Q^{\pi}$ is ULC with Lipschitz constant $L_{Q}$ and the policy $\pi$ is ULC with Lipschitz constant $L_{\pi}$, then $V^{\pi}$ is ULC with Lipschitz constant $L_Q(1+L_{\pi})$.
\end{corollary}

\begin{lemmaap}
We assume that the next state probability measure $\prob^a_s$ admits a density $p^a_s: \sspace \rightarrow [0, \infty)$ with respect to the Lebesgue measure.
If $s' \mapsto p^a_s(s')$ is bounded and $\rew: \sspace \times \aspace \rightarrow \rR$ and $s \mapsto p^a_s(s')$ are LC, then $Q^{*}: \sspace \times \aspace \rightarrow \rR$ is LC. 
\end{lemmaap}
We start by recalling the dominated convergence theorem:
\begin{theorem} (Lebesgue's Dominated Convergence Theorem)\\
 Let $\{f_n\}$ be a sequence of complex-valued measurable functions on a measure space $(\sspace,\Sigma, \mu)$. Suppose that:
\begin{enumerate}
  \item the sequence $\{f_n\}$ converges pointwise to a function $f$
  \item the sequence $\{f_n\}$ is dominated by some integrable function $g$, that is, $\forall n \in \mathbb{N}, \forall x \in \sspace$, $|f_{n}(x)|\leq g(x)$ 
\end{enumerate}
Then, $f$ is integrable and $\lim_{n \to \infty}\int_{\sspace}f_n(d\mu)=\int_{\sspace}f d\mu$.
\end{theorem}
Let's define the following sequence:
$Q_0(s, a) = 0$ and $Q_{n+1}(s, a) = \rew(s, a) + \gamma \mathop\mathbb{E}_{\substack{s' \sim\prob(\cdot | s, a)}}[\max_{a' \in \aspace} Q_n(s', a')]$ for all $s \in \sspace, a \in \aspace$. \\
$Q_0$ is constant and thus continuous on $\sspace \times \aspace$.\\
To show that the continuity of $Q_{n}$ implies the continuity of $Q_{n+1}$, let's apply the dominated convergence theorem. 
\begin{enumerate}
\item  $s \mapsto p^a_s(s')$ being continuous, $s_m$ tends to $s \in \sspace$ implies that $\max\limits_{a' \in \aspace} [Q_n(s', a')]p^a_{s_m}(s')$ tends to $\max\limits_{a' \in \aspace} [Q_n(s', a')]p^a_{s}(s')$.
\item For all $s' \in \sspace, |\max\limits_{a'}Q_n(s', a')|<V_{max}$ by assumption.\\
We fix $a \in \aspace$. $p^a_s$ is bounded so there exists a function $h_a \in L^1(\sspace)$ such that $p^a_{s}(s') \leq h_a(s')$ for all $s', s.$ 
\end{enumerate}
By the dominated convergence theorem, $\lim_{m \to \infty}\bigints_{\sspace} \max\limits_{a' \in \aspace} [Q_n(s', a')]p^a_{s_m}(s')ds'=\bigints_{\sspace}\max\limits_{a' \in \aspace} [Q_n(s', a')]p^a_{s}(s') ds'$.
That is, if $s_m \rightarrow s$, then $\mathop\mathbb{E}_{\substack{s' \sim\prob(\cdot | s_m, a)}}[\max_{a' \in \aspace} Q_n(s', a')] \rightarrow \mathop\mathbb{E}_{\substack{s' \sim\prob(\cdot | s, a)}}[\max_{a' \in \aspace} Q_n(s', a')]$.\\
The reward function being LC by assumption, $Q_{n+1}$ is LC.\\ \\
Let's now show that $Q^* = \text{lim}_{n \rightarrow \infty}Q_n$ is LC.\\
Let $T: \left(C(\sspace \times \aspace), ||\cdot||_{\infty}\right) \rightarrow \left(C(\sspace \times \aspace), ||\cdot||_{\infty}\right)$, where $C(\sspace \times \aspace) =\mathcal{X}$ is the space of LC functions on $\sspace \times \aspace$ and $||\cdot||_{\infty} = \text{sup}_{(s, a) \in \sspace \times \aspace}$, be defined by:\\
$(Tf)(s, a)=  \rew(s, a) + \gamma \mathop\mathbb{E}_{\substack{s' \sim\prob(\cdot | s, a)}}[\max_{a' \in \aspace} f(s', a')]$. It is known that $||Tf - Tg||_{\infty} \leq \gamma ||f - g||_{\infty}$. The contraction mapping theorem implies that $TQ_{n}=Q_{n+1} \rightarrow Q^*$ in $\mathcal{X}$ (sup norm), so $Q^* \in \mathcal{X}$, that is $Q^*$ is LC.\\ \\
The previous result can be stated more generally as follows:
\begin{corollary}
We assume that for each $a \in \aspace$, $\prob^a_s$ is (weakly) continuous as a function of $s$, that is, if $s_n$ converges to $s\in \sspace$ then for every bounded continuous function $f:\sspace \rightarrow \rR$, $\bigintssss f d\prob^a_{s_n}$ tends to $\bigintssss f d\prob^a_s$. If $\rew: \sspace \times \aspace \rightarrow \rR$ is LC then $Q^{*}: \sspace \times \aspace \rightarrow \rR$ is LC. 
\label{general}
\end{corollary}
The proof of this result is similar as above but does not involve the Dominated Convergence Theorem as the continuity of $\mathop\mathbb{E}_{\substack{s' \sim\prob(\cdot | s, a)}}[\max_{a' \in \aspace} Q_n(s', a')]$ as a function of $s$ is ensured by the assumption of weakly continuity of $\prob^a_s$.\\
We note that the assumption of weakly continuity of $\prob^a_s$ is weaker than the one on the existence of a density $p^a_s$ as above. Indeed, if $s \mapsto p^a_s$ is continuous, then $s \mapsto \prob^a_s$ is weakly continuous by Scheffé's lemma. 
\\
For the ULC case, the conditions and proof under which the implication "$\prob + \rew \implies Q^*$" hold are stated by in \citet{gelada2019deepmdp}:
\begin{corollary}[\citeauthor{gelada2019deepmdp}, \citeyear{gelada2019deepmdp}]
If $\rew$ and $\prob$ are ULC with Lipschitz constant $L_{\rew}$ and $L_{\prob}$ and $\gamma L_{\prob}<1$, then $Q^{*}$ is ULC with Lipschitz constant $\frac{L_{\rew}}{1-\gamma L_{\prob}}$.
\end{corollary}
The reverse implication "$Q^* \implies \rew + \prob$" is not true as shows the following counter-example. Let's suppose $\sspace=\rR$ and $\aspace=\{1, 2\}$. Let $\rew(s, 1)=1$ and let $\rew(s, 2)$ be any discontinuous function, for instance:
 \begin{align}
\rew^{*}(s, 2) & = \left\{
    \begin{array}{ll}{}
        0 & \mbox{if } s\leq s_0, s_0 \in \rR  \\
        0.5 & \mbox{if } s>s_0.
    \end{array}
\right.
\end{align}
This leads to $Q^*(s, a)= \frac{1}{1-\gamma}$ which is continuous but $\rew(s, 2)$ is discontinuous.

\begin{lemmaap}
We assume that the next state probability measure $\prob^a_s$ admits a density $p^a_s: \sspace \rightarrow [0, \infty)$ with respect to the Lebesgue measure.
If $s' \mapsto p^a_s(s')$ is bounded and $\rew: \sspace \times \aspace \rightarrow \rR$, $\pi: \sspace \rightarrow \Delta(\aspace)$ and $s \mapsto p^a_s(s')$ are LC, then $Q^{\pi}: \sspace \times \aspace \rightarrow \rR$ is LC. 
\end{lemmaap}
Similarly, we proceed by induction and consider the following sequence:
\begin{equation*}
Q^{\pi}_{0}(s, a) = 0
\end{equation*}
\begin{equation*}
Q^{\pi}_{n+1}(s, a) = \rew(s, a) + \gamma \mathop\mathbb{E}_{\substack{s' \sim \prob^a_s}}V^{\pi}_n(s')
\end{equation*}
As above, $Q^{\pi}_{0}$ is continuous and we then proceed by induction and apply the dominated convergence theorem. \\
We assume that $Q^{\pi}_n$ is continuous. Assuming $\pi$-continuous, we have shown that this also implies that $V^{\pi}_n$ is continuous.
\begin{enumerate}
\item $s \mapsto p^a_s(s')$ being continuous, $s_m$ tends to $s \in \sspace$ implies that $p^a_{s_m}(s')V_n^{\pi}(s')$ tends to $p^a_{s}(s')V_n^{\pi}(s')$.
\item For all $s' \in \sspace, |V^{\pi}(s')|<V_{max}$ by assumption.\\
We fix $a \in \aspace$. $p^a_s$ is bounded so there exists a function $h_a \in L^1(\sspace)$ such that $p^a_{s}(s') \leq h_a(s')$ for all $s', s.$ 
\end{enumerate}
By the dominated convergence theorem, $\lim_{m \to \infty}\bigints_{\sspace}p^a_{s_m}(s')V_n^{\pi}(s')ds'=\bigints_{\sspace}p^a_{s}(s') V_n^{\pi}(s')ds'$.
Hence, if $s_m \rightarrow s$, then $\mathop\mathbb{E}_{\substack{s' \sim\prob(\cdot | s_m, a)}}V_n^{\pi}(s') \rightarrow \mathop\mathbb{E}_{\substack{s' \sim\prob(\cdot | s, a)}} V_n^{\pi}(s')$.\\
The reward function being LC by assumption, $Q^{\pi}_{n+1}$ is LC.\\ \\
As shown above, $Q^{\pi}_n$ converges to $Q^{\pi}$ in sup norm, so $Q^{\pi}$ is continuous.
\begin{corollary}
We assume that for each $a \in \aspace$, $\prob^a_s$ is (weakly) continuous as a function of $s$, that is, if $s_n$ converges to $s\in \sspace$ then for every bounded continuous function $f:\sspace \rightarrow \rR$, $\bigintssss f d\prob^a_{s_n}$ tends to $\bigintssss f d\prob^a_s$. If $\rew: \sspace \times \aspace \rightarrow \rR$ and $\pi:\sspace \rightarrow \Delta(\aspace)$ are LC then $Q^{\pi}: \sspace \times \aspace \rightarrow \rR$ is LC. 
\label{general}
\end{corollary}
The proof of this result is similar as above but does not involve the Dominated Convergence Theorem as the continuity of $\mathop\mathbb{E}_{\substack{s' \sim\prob(\cdot | s, a)}}V^{\pi}(s')$ as a function of $s$ is ensured by the assumption of weakly continuity of $\prob^a_s$.\\
For the ULC case, the conditions and proof under which the implication "$\prob + \rew \implies Q^{\pi}$" hold are stated by in \citet{rachelson2010locality}:
\begin{lemmaap}[\citeauthor{rachelson2010locality}, \citeyear{rachelson2010locality}]
If $\rew$, $\prob$ and $\pi$ are ULC with Lipschitz constant $L_{\rew}$, $L_{\prob}$ and $L_{\pi}$, and if $\gamma L_{\prob}(1+L_{\pi})<1$, then $Q^{\pi}$ is ULC with Lipschitz constant $\frac{L_{\rew}}{1-\gamma L_{\prob}(1+L_{\pi})}$.
\end{lemmaap}
We note that the reverse implication "$Q^{\pi} \implies \rew + \prob$" is not true as there exists a class of policies (the optimal policy is one element of this class) for which the implication does not hold.

\section{Proofs of the section Taxonomy of metrics}
\idmetric*
\begin{proof}
Recall that for any (pseudo-)metric space $(X, d)$, a set $U \subseteq X$ is open if for any $x \in U$, there exists $r > 0$ such that the open ball $B_d(x, r)$ of radius $r$ centered at $x$ is a subset of $U$.\\
Suppose $U \subseteq \sspace$ is a non-empty open set of $\mathcal{S}$. Then for any $x \in U$, there exists $r > 0$ such that $B_{e^{\rI}}(x, r)= \{y \in \sspace| e^{\rI}(x, y)<r\} \subseteq U$. If $r > 1$, $B_{e^{\rI}}(x, r)=\sspace$ and $U = \sspace$. Hence, $\sspace \subset \mathcal{T}_{e^{\rI}}$. Else, $B_{e^{\rI}}(x, r)=\{x\}$ and $\{x\} \subset U$. This is true for all $x \in U$ so $\cup_{x \in U}\{x\} \subset \mathcal{T}_{e^{\rI}}$. Hence, $\mathcal{T}_{e^{\rI}}$ is the collection of all open subsets of $\sspace$, that is, it is the discrete topology on $\sspace$.
\\ \\
Let $(Y, d_Y)$ be any metric space.
\begin{itemize}
 \item We first show that any function $h:(\sspace, e^{\rI}) \rightarrow (Y, d_Y)$  is LC and UC.
Let $\epsilon > 0$.\\
We choose $\delta=\frac{1}{2}.$ Then, for all $x, y \in \sspace$,
\begin{align*}
e^{\rI}(x, y)<\delta &\implies e^{\rI}(x, y) =0  \text{ as } e^{\rI}: \sspace \times \sspace \rightarrow \{0, 1\}.\\
 &\implies x=y \text{ as } e^{\rI} \text{ is a proper metric}.\\
 & \implies d_{Y}(h(x), h(y))=0 \text{ as } d_Y \text{ is a pseudometric.} \\
  & \implies d_{Y}(h(x), h(y))< \epsilon.
  \end{align*}
This shows that any $h$ is UC and thus LC.
\item We now show that any  bounded function $h:(\sspace, e^{\rI}) \rightarrow (Y, d_Y)$  is ULC.\\
$h$ is ULC if there exist $K>0$ such that for all $x, x' \in X$ we have $d_Y(h(x), h(x')) \leq Kd_{\sspace}(s, s')$.\\
If $s=s'$, then $d_{\sspace}(s, s')=0$ by definition of $e^{\rI}$ and $h(s)=h(s')$ which implies $d_Y(h(s), h(s'))=0$. Hence, $h$ is ULC.\\
Else, $e^{\rI}(s, s')=1$. $h$ is bounded so $h(\sspace)$ is a bounded subset of $Y$, that is for all $s, s' \in \sspace$, $d_Y(h(s), h(s'))\leq c$ for some $c>0$. Hence, $h$ is lipschitz.
\end{itemize}
\end{proof}
\trivialmetric*
\begin{proof}
  For $e^{\rT}$, suppose $U \subseteq \mathcal{S}$ is a non-empty open set of $\mathcal{S}$. Then for any $x \in U$, there exists $r > 0$ such that $B_{e^{\rT}}(x, r) \subseteq U$. But observe that, for all $r > 0, x\in \mathcal{S}$, we have $B_{e^{\rT}}(x, r) = \mathcal{S}$ by definition of the trivial pseudo-metric $e^{\rT}$. Hence $U = \mathcal{S}$ and $(\mathcal{S}, e^{\rT})$ has the trivial topology. 
\\ \\
 Let $(Y, d_Y)$ be any metric space. Any function $h:(\sspace, e_{\rT}) \rightarrow (Y, d_Y)$ is LC (resp UC, resp ULC) iff $h$ is constant.\\
$(\impliedby):$ Suppose $h$ is constant taking value $y \in Y$. Recall that $ e_{\rT}: \sspace \times \sspace \rightarrow \{0, 1\}.$ It is clear that $h$ must be Lipschitz continuous, and thus uniformally and locally continuous, because any $K>0$ satisfies for all $s, s' \in \sspace,$ $d_Y(h(s), h(s')) = d_Y(y, y)= 0\leq K e_{t}(x, y)=0$\\
$(\implies):$ Suppose for sake of contradiction that $h: \mathcal{S} \rightarrow Y$ is LC but not constant. Then there exist $s_1, s_2 \in \mathcal{S}$ such that $h(s_1) \neq h(s_2)$. This means that there exists $\epsilon_0>0$ such that $d_Y(h(x)-h(y))>\epsilon_0$. Because we are using the trivial metric on $\sspace$, $d(x, y)=0<\delta$ for all $\delta>0$. This contradicts the LC assumption of $h$. Hence, $h$ LC implies that $h$ is constant. This reasoning also holds for UC (resp ULC) as a function that cannot be LC cannot be UC (resp. ULC) (since ULC $\implies$ UC $\implies$ LC).
 \end{proof}

\agg*
 As a byproduct, we can note that the metric $e^{\phi}$ induces the finest topology on $\hat{\sspace}$. Indeed, by definition, $e_{\phi}:\sspace \rightarrow \{0, 1\}$ is equal to the discrete pseudometric $e^{\rI}:\hat{\sspace} \rightarrow  \{0, 1\}$. Thanks to ~\autoref{the:DiscreteMetricCont}, we can deduce that $e_{\phi}$ induces the discrete topology on $\hat{\sspace}$.
\begin{proof}
\begin{itemize}
\item We first show that any function $f: (\sspace, d_{\sspace}) \rightarrow (Y, d_Y)$ is UC (and thus LC) with respect to the metric $e^{\phi_{f, 0}}$.

Let $\epsilon >0$. We choose $\delta = \frac{1}{2}$. Then, for all $s, t \in \sspace,$ $e^{{\phi}_{f, 0}}(s, t)(s, t)<\delta \implies e^{\phi_{f, 0}}(s, t)=0 \implies \phi(s)=\phi(t) \implies f(s)=f(t) \implies |f(s)-f(t)|\leq \epsilon.$
\item Then, the fact that any function $f: (\sspace, d_{\sspace}) \rightarrow (Y, d_Y)$ is ULC with respect to the pseudometric $e^{\phi_{f, 0}}$ iff $f$ is bounded is a consequence of ~\autoref{the:DiscreteMetricCont}
\item Finally, if $\eta>0$, there is no continuity garantee about $f: (\sspace, d_{\sspace}) \rightarrow (Y, d_Y)$ with respect to the metric $e^{\phi_{f, \eta}}$.\\
Let $\epsilon >0$. We choose $\delta = \frac{1}{2}$. Then, $e^{\phi_{f, \eta}}(s, t)<\delta \implies e^{\phi_{f, \eta}}(s, t)=0 \implies \phi(s)=\phi(s') \implies |f(s)-f(s')|\leq \eta.$\\
If $\epsilon \geq \eta$, then the continuity definition is respected.\\
But when $\epsilon < \eta$, we cannot conclude anything.
\end{itemize}
\end{proof}

\bisimcont*
\begin{proof}
  Take any $s,t\in\cS$ and $a\in\cA$.
  \begin{align*}
    |Q^*(s, a) - Q^*(t, a)| & = \left|\cR_s^a + \gamma\sum_{s'\in\cS}\cP_s^a(s')V^*(s') - \left(\cR_t^a + \gamma\sum_{t'\in\cS}\cP_t^a(t')V^*(t')\right)\right| \\
    & = \left|\cR_s^a -\cR_t^a + \gamma\sum_{s'\in\cS}V^*(s')(\cP_s^a(s') - \cP_t^a(s'))\right| \\
    & \leq |\cR_s^a -\cR_t^a| + \gamma\left|\sum_{s'\in\cS}V^*(s')(\cP_s^a(s') - \cP_t^a(s'))\right| \\
    & \leq |\cR_s^a -\cR_t^a| + \gamma\cW(d^{\sim})(\cP_s^a(s'), \cP_t^a(s')) \\
    & \leq \max_{a\in\cA}\left\{ |\cR_s^a -\cR_t^a| + \gamma\cW(d^{\sim})(\cP_s^a(s'), \cP_t^a(s')) \right\} \\
    & = d^{\sim}(s, t)
  \end{align*}
  We can show the result for $Q^{\pi}$ similarly.
\end{proof}

\bisimcontbis*
\begin{proof}
This is a consequence of \autoref{the:bisimCont}, the normalization coming from the fact that the bisimulation metric $d^{\sim}$ is bounded by $\frac{R_{\text{max}}}{1-\gamma}$.
\end{proof}

\Deltametrics*
\begin{proof}
Let's show that these functions are pseudometrics.\\
First, $d_{\Delta^*}(s, s)=0$.\\
Second, $d_{\Delta^*}(s, s')=d_{\Delta^*}(s', s)$ by symmetry of the graph of the absolute value function.
\begin{align*}
\text{Finally}, d_{\Delta^*}(s_1, s_2) &=\max\limits_{\aspace}|Q^*(s_1, a) - Q^*(s_3, a) + Q^*(s_3, a) - Q^*(s_2, a)|\\ 
&\leq \max\limits_{\aspace}\left(|Q^*(s_1, a) - Q^*(s_3, a)| + |Q^*(s_3, a) - Q^*(s_2, a)|\right), \text{by the triangular inequality.}\\
&\leq \max\limits_{\aspace}|Q^*(s_1, a) - Q^*(s_3, a)|) + \max\limits_{\aspace}|Q^*(s_3, a) - Q^*(s_2, a)|\\
& = d_{\Delta^*}(s_1, s_2) + d_{\Delta^*}(s_1, s_2).
\end{align*}
Hence, $d_{\Delta^*}$ satisfies the triangular inequality. We can thus conclude that $d_{\Delta^*}$ is a pseudometric.\\ \\
Similarly, we prove that $d_{\Delta^{\pi}}$ and $d_{\Delta_{\forall}}$ are pseudometrics.\\
We now prove the continuity properties given by these three metrics.

\begin{itemize}
\item 
Let $s, t \in \sspace$.
As above, we fix $a \in \aspace$. Then,
$|Q^*(s, a)-Q^*(t, a)|\leq \max\limits_{a \in \aspace}|Q^*(s, a)-Q^*(t, a)|=d_{\Delta^*}(s, t)$. Thus, $Q^*$ is lipschtiz continuous with respect to $d_{\Delta^*}$.

\item 
Let $s, t \in \sspace$ and let $\pi \in \Pi$. As before, let's fix $a \in \aspace$.
$|Q^{\pi}(s, a)-Q^{\pi}(t, a)|\leq \max\limits_{a \in \aspace}|Q^{\pi}(s, a)-Q^{\pi}(t, a)|=d_{\Delta_{\pi}}(s, t)$. Thus, $Q^{\pi}$ (resp.$Q^{*}$) is lipschtiz continuous with respect to $d_{\Delta_{\pi}}$ (resp. $d_{\Delta_{\pi^*}}$).\\

\item
Let $s, t \in \sspace$ and let $\pi \in \Pi$. As before, let's fix $a \in \aspace$.\\
$|Q^{\pi}(s, a)-Q^{\pi}(t, a)|\leq \max\limits_{\pi \in \Pi}|Q^{\pi}(s, a)-Q^{\pi}(t, a)| \leq \max\limits_{a \in \aspace, \pi \in \Pi}|Q^{\pi}(s, a)-Q^{\pi}(t, a)| = d_{\Delta_{\pi}}(s, t)$.\\
This in particular true for any $\pi \in \Pi$, hence for any policy $\pi$, $Q^{\pi}$ is lipschtiz continuous with respect to $d_{\Delta_{\forall}}$.
\end{itemize}
\end{proof}
We now formalize in \autoref{lemma:laxlip} and \autoref{lemma:laxlipbis} the results from \citet{taylor2009bounding} that we added to \autoref{fig:largertable}.
\begin{lemmaap}\label{lemma:laxlip}
$\forall s, s' \in \sspace, |V^*(s)-V^*(s')| \leq d^{\sim_{lax}}(s, s') $
\end{lemmaap}
\begin{proof}
The proof of this result can be found in \citep{taylor2009bounding}.
\end{proof}
\begin{lemmaap}\label{lemma:laxlipbis}
$\forall s, s' \in \sspace, \frac{1 - \gamma}{R_{max}}|V^*(s)-V^*(s')| \leq e^{\sim_{lax}}(s, s') $
\end{lemmaap}
\begin{proof}
This result is a consequence of Lemma \ref{lemma:laxlip}. The normalization comes from the fact that the lax bisimulation metric $d^{\sim_{lax}}$ is bounded by $\frac{R_{\text{max}}}{1-\gamma}$.
\end{proof}

\avfmetrics*
\begin{proof}
The space of value functions $\{V^{\pi}|\pi \in \Pi\}$ is a polytope \citep{dadashi2019value} and \cite{bellemare19geometric} considered the finite set of policies $\Pi_{\text{AVF}}$ corresponding to extremal verticies of this polytope $\{V^{\pi}|\pi \in \Pi_{\text{AVF}}\}$

As noted by \cite{dabney2020value}, the space of action value functions $\{Q^{\pi}|\pi \in \Pi\}$ is also polytope since polytopes are invariant by translations (reward $\rew_s^a$ term) and linear transformations ($\gamma \trans_s^a$ term). Additionally, extremal vertices of $\{Q^{\pi}|\pi \in \Pi\}$ and $\{V^{\pi}|\pi \in \Pi\}$ are reached for the same policies as extremal points of a polytope are invariant by affine transformations. Hence, the set of extremal vertices of $\{Q^{\pi}|\pi \in \Pi\}$ is $\{Q^{\pi}|\pi \in \Pi_{\text{AVF}}\}$.
The maximum between two elements of this polytope is reached at two extremal vertices of the polytope, hence the result.

Now, when approximating the metric $d_{\Delta_{\forall}}$ by $d_{\widetilde{\text{AVF}}(n)}=\max\limits_{\pi\in \Pi_{\widetilde{\text{AVF}}(n)}, a\in\aspace}|Q^{\pi}(s, a)-Q^{\pi}(s', a)|$, where $\Pi_{\widetilde{\text{AVF}}(n)}$ are $n$ samples from the set of extremal policies $\Pi_{\text{AVF}}$, it follows that $d_{\widetilde{\text{AVF}}(n)} \leq d_{\Delta_{\forall}}$.
\end{proof}

\begin{customthm}{2}
\label{the:topologyships}
  The relationships between the topologies induced by the metrics in \autoref{fig:largertable} are given by the following diagram.  We denote by $d_1\rightarrow d_2$ when $\mathcal{T}_{d_1} \subset  \mathcal{T}_{d_2}$, that is, when $\cT_{d_1}$ is coarser than $\cT_{d_2}$. Here $d$ denotes any arbitrary metric.
  \begin{equation*}
    \begin{tikzcd}
      e^{\sim_{lax}} \arrow[d] &d^{\sim_{lax}} \arrow[l]  \arrow[d] & & d  \arrow[r]  &e^{\rI} \\
      e^{\sim}  & d^{\sim}\arrow[l]& d_{\Delta^*}\arrow[l]\arrow[r] & d_{\Delta_{\forall}} &  d_{\Delta_{\pi}}\arrow[l] \arrow[d]&&& & & \\
      d& e^{\rT} \arrow[l] & & e^{\sim_{\pi}}& d^{\sim_{\pi}} \arrow[l]& 
    \end{tikzcd}
    \label{diag:metricRel}
  \end{equation*}
\end{customthm}

We now prove the second theorem of our paper, Theorem ~2. The proof itself will be made up of a series of lemmas for each of the arrows in the diagram. We first start by proving a necessary lemma.
\begin{lemmaap}
Given two metrics $d_1$ and $d_2$ on $\sspace$, if there exists $\alpha >0$ such that $d_1(s, t) \leq \alpha d_2(s, t)$ for all $s, t \in \sspace$, then $\mathcal{T}_{d_1}$ is coarser than $\mathcal{T}_{d_2}$, that is $\mathcal{T}_{d_1} \subset \mathcal{T}_{d_2} $ .
\end{lemmaap}
\begin{proof}
Let $\epsilon >0$ and $x \in \sspace$.
Suppose $x' \in B_{d_2}(x, \epsilon)$. By definition, this means that $d_2(x, x')\leq \epsilon$. It implies $\frac{1}{\alpha}d_1(x, x')\leq d_2(x, x')\leq \epsilon$ by assumption and then $x' \in B_{d_1}(x,\alpha\epsilon)$. Hence, $B_{d_2}(x, \epsilon) \subset B_{d_1}(x, \alpha \epsilon)$.\\
Now, suppose $U \subset \sspace$ is a non-empty open set of $\sspace$.Then $\forall x \in U$,there exists $r>0$ such that $B_{d_1}(x, r) \subset U$. We have shown that $B_{d_2}(x, \epsilon) \subset B_{d_1}(x, \alpha\epsilon)$ for all $\epsilon >0$. So we also have $B_{d_2}(x, \frac{r}{\alpha}) \subset U$. By definition a topology on $\sspace$ is a collection of open subsets on $\sspace$ so we can conclude that $\mathcal{T}_{d_1} \subset \mathcal{T}_{d_2} $.
\end{proof}

\begin{lemmaap}
For all $s, t \in \sspace, e^{\rT}(s, t) \leq  e^{\rI}(s, t)$ and $\mathcal{T}_{e^{\rT}} \subset \mathcal{T}_{e^{\rI}} $. 
\end{lemmaap}
\begin{proof}
This comes directly from the definitions of the trivial metric $e^{\rT}$ and discrete metric $e^{\rI}$. Moreover, as mentioned in \autoref{the:DiscreteMetricCont} and \autoref{the:DiscreteMetricContt}, the discrete metric induces the finest topology on $\sspace$ while the trivial metric induces the coarsest topology.
\end{proof}
\begin{lemmaap}
For all $s, t \in \sspace,d_{\Delta^*}(s, t) \leq d_{\Delta}(s, t)$ and $d_{\Delta_{\pi}}(s, t) \leq d_{\Delta}(s, t)$.
\end{lemmaap}
\begin{proof}
By definition, for all $s, s' \in \sspace$, $d_{\Delta^*}(s, s')=\max\limits_{a\in\aspace}|Q^*(s, a)-Q^*(s', a)|\leq \max\limits_{\pi\in\Pi, a\in\aspace}|Q^{\pi}(s, a)-Q^{\pi}(s', a)|$ so $d_{\Delta^*}(s, t) \leq d_{\Delta}(s, t)$.\\
The proof is similar for $d_{\Delta_{\pi}}$
\end{proof}
\begin{lemmaap}
For all $s, t \in \sspace$, there exists $\alpha>0$ such that $d^{\sim}(s, t) \leq \alpha e^{\sim}(s, t)$ and $d^{\sim_{lax}}(s, t) \leq \alpha e^{\sim_{lax}}(s, t)$ and $d^{\sim_{\pi}}(s, t) \leq \alpha e^{\sim_{\pi}}(s, t)$.
\end{lemmaap}
\begin{proof}
Let $s, t \in \sspace$.\\
If $e^{\sim}(s, t)=0$, then $s \sim t$ and $d^{\sim}(s, t)=0$.\\
If $e^{\sim}(s, t)=1$, then $d^{\sim}(s, t)\neq0$. Moreover, by construction of the bisimulation metric, $d^{\sim}(s, t)\in [0, \frac{R_{\text{max}}}{1-\gamma}]$ for all $s, t\in \sspace$. Hence, $d^{\sim}(s, t) \leq \frac{R_{\text{max}}}{1-\gamma}e^{\sim}(s, t)$. Choosing $\alpha=  \frac{R_{\text{max}}}{1-\gamma}$, we get $d^{\sim}(s, t) \leq \alpha e^{\sim}(s, t)$.\\
The proof is similar for the two other inequalities.
\end{proof}
\begin{lemmaap}
For all $s, t \in \sspace, e^{\sim_{lax}}(s, t) \leq e^{\sim}(s, t)$ and $d^{\sim_{lax}}(s, t) \leq d^{\sim}(s, t)$.
\end{lemmaap}
\begin{proof}
Let $s, t \in \sspace$.\\
If $e^{\sim_{lax}}(s, t)=0$ then the inequality is verified by postivity of any metric.\\
If $e^{\sim_{lax}}(s, t)=1$, it means there exists $a \in \aspace$ such that for all $b \in \aspace$, $\rew^a_s \neq \rew^b_s$ and there exists $X \in \Sigma(E)$ such that $\prob^a_s(X) \neq \prob^a_b(X)$. This is in particular true when $b=a$ which means that the conditions to be a bisimulation relation are not satisfied. Hence, $e^{\sim}(s, t)=1$. \\
The second inequality is proven in \citet{taylor2009bounding}.
\end{proof}
\textbf{Complexity results}\\
We now provide details explaining the complexity results from \autoref{tab:properties}.\\
The discrete identity metric $e^\rI$ compares all pairs of states which results in a complexity $O(|\sspace|)$. The complexity of the trivial metric $e^\rT$ is independent of the number of states and hence constant.

The discrete bisimulation metric can be computed by finding the bisimulation equivalence classes. This can be done by starting with a single equivalence class that gets iteratively split into smaller equivalence classes when one of the bisimulation conditions is violated; this process is repeated until stability. Each iteration of this process is $O(|\aspace| |\sspace|^2)$, since we are performing an update for all actions and pairs of states. Since there can be at most $|\sspace|$ splits, this yields the complexity of $O(|\aspace| |\sspace|^3)$.

The bisimulation metric can be computed by iteratively applying $\lfloor \frac{\ln \delta}{\ln \gamma} \rfloor$ times the operator F from \autoref{lemmacontbisim} \citep{ferns2004metrics}, for each action and each pair of states. We note that the time complexity for solving an optimal flow problem as originally presented by \citet{ferns2004metrics} is incorrect and off by a factor of $|\sspace|$: it should be $O(|\sspace|^3\log |\sspace|)$. Thus, the complexity they present for computing the bisimulation metric is also off by a factor of $|\sspace|$; the corrected time complexity is $O\big(|\cA||\cS|^5\log |\cS|\frac{\ln\delta}{\ln \gamma}\big)$, as we presented in Table~1. The $\pi$-bisimulation metrics do not require a loop over the action space as the matching is under a fixed policy. Therefore their complexity is the one of the bisimulation metrics off by a factor $|\aspace|$.

The time complexity of computing $d_{\Delta_{\pi}}$ and $d_{\Delta^{*}}$ is the same as the complexity of policy evaluation and value iteration, plus an extra $O(|\sspace^2|)$ term for computing the resulting metric; this last term is dominated by the policy/value iteration complexity, which is what we have included in the table.

 \section{Formal definition of bisimulation metrics}

We will also define a metric between probability functions that is used by some of the state metrics considered in this paper.
Let $(Y, d_Y)$ be a metric space with Borel $\sigma$-algebra $\Sigma$.

\begin{definition}
  The \textbf{Wasserstein} distance \citep{villani2008optimal} between two probability measures $P$ and $Q$ on $Y$, under a given metric $d_Y$ is given by
  $W_{d_Y}(P, Q) = \inf_{\lambda \in \Gamma(P, Q)}\mathop\mathbb{E}_{(x, y)\sim \lambda}[d_Y(x, y)]$, where $\Gamma(P, Q)$ is the set of couplings between $P$ and $Q$.
\end{definition}

The Wasserstein distance can be understood as the minimum cost of transporting $P$ into $Q$ where the cost of moving a unit mass from the point $x$ to the point $y$ is given by $d(x, y)$.
 \begin{theorem}
  \label{thm:pi_bisim_operator}
  Define $\cF^{\pi}:\mathcal{M}\rightarrow\mathcal{M}$ by
  $\cF^{\pi}(d)(s, t) = |\cR^{\pi}_s - \cR^{\pi}_{t}| + \gamma
  \cW_1(d)(\cP^{\pi}_s, \cP^{\pi}_{t})$, 
  then $\cF^{\pi}$ has a least fixed point $d^{\pi}_{\sim}$, and
  $d^{\pi}_{\sim}$ is a $\pi$-bisimulation metric.
\end{theorem}

\begin{definition}
  Given a $1$-bounded pseudometric $d\in\rM$, the metric
  $\delta(d):\cS\times\cA\rightarrow [0,1]$ is defined as follows:
  \[ \delta(d)((s, a), (t, b)) = |\cR(s, a) - \cR(t, b)| + \gamma\cW(d)(\cP(s, a), \cP(t, b)) \]
\end{definition}

\begin{definition}
  Given a finite $1$-bounded metric space $(\mathfrak{M}, d)$ let
  $\mathcal{C}(\mathfrak{M})$ be the set of compact spaces (e.g. closed and
  bounded in $\rR$). The {\em Hausdorff metric}
  $H(d):\mathcal{C}(\mathfrak{M})\times\mathcal{C}(\mathfrak{M})\rightarrow [0,
  1]$ is defined as:
  \[ H(d)(X, Y) = \max\left(\sup_{x\in X}\inf_{y\in Y} d(x, y), \sup_{y\in Y}\inf_{x\in X} d(x, y)\right) \]
\end{definition}

\begin{definition}
  Denote $X_s=\lbrace (s, a) | a\in\cA\rbrace$. We define the operator
  $F:\rM\rightarrow\rM$ as:
  \[ F(d)(s, t) = H(\delta(d))(X_s, X_t) \]
\end{definition}

\begin{theorem}
  $F$ is monotonic and has a least fixed point $d_{lax}$ in which
  $d_{lax}(s, t) = 0$ iff $s\sim_{lax} t$.
\end{theorem}

\section{Additional empirical evaluations}
In this section we conduct an empirical evaluation to complement the theoretical analyses performed above. We conduct these experiments on the well-known four-room domain \citep{sutton99between,solway14optimal,machado17laplacian,bellemare19geometric} for all our experiments, which is illustrated in \autoref{fig:4rooms}. This domain enables clear visualization and ensures that we can compute a metric defined over the entirety of the state space. These experiments aim to showcase 1) the qualitative difference in the state-wise distances produced by the different metrics; 2) visualize the differences in abstract states that the different metrics produce when used for state aggregation.

The dynamics of the environment are as follows. There are four actions (up, down, left, right), transitions are deterministic, there is a reward of $+1$ upon entering the non-absorbing goal state, there is a penalty of $-1$ for running into a wall, and we use a discount factor $\gamma=0.9$. We will conduct our experiments on four representative metrics: $d^{\sim}$, $d^{\sim_{lax}}$, $d_{\Delta^*}$, and $d_{\Delta}$. Since the maximization over all policies required for $d_{\Delta}$ is in general intractable, we instead sample 50 adversarial value functions (AVFs) \citep{bellemare19geometric} as a proxy for the set of all policies. 

\clearpage
\subsection{Four rooms}

\begin{figure}[!h]
  \centering
  \includegraphics[width=0.5\textwidth]{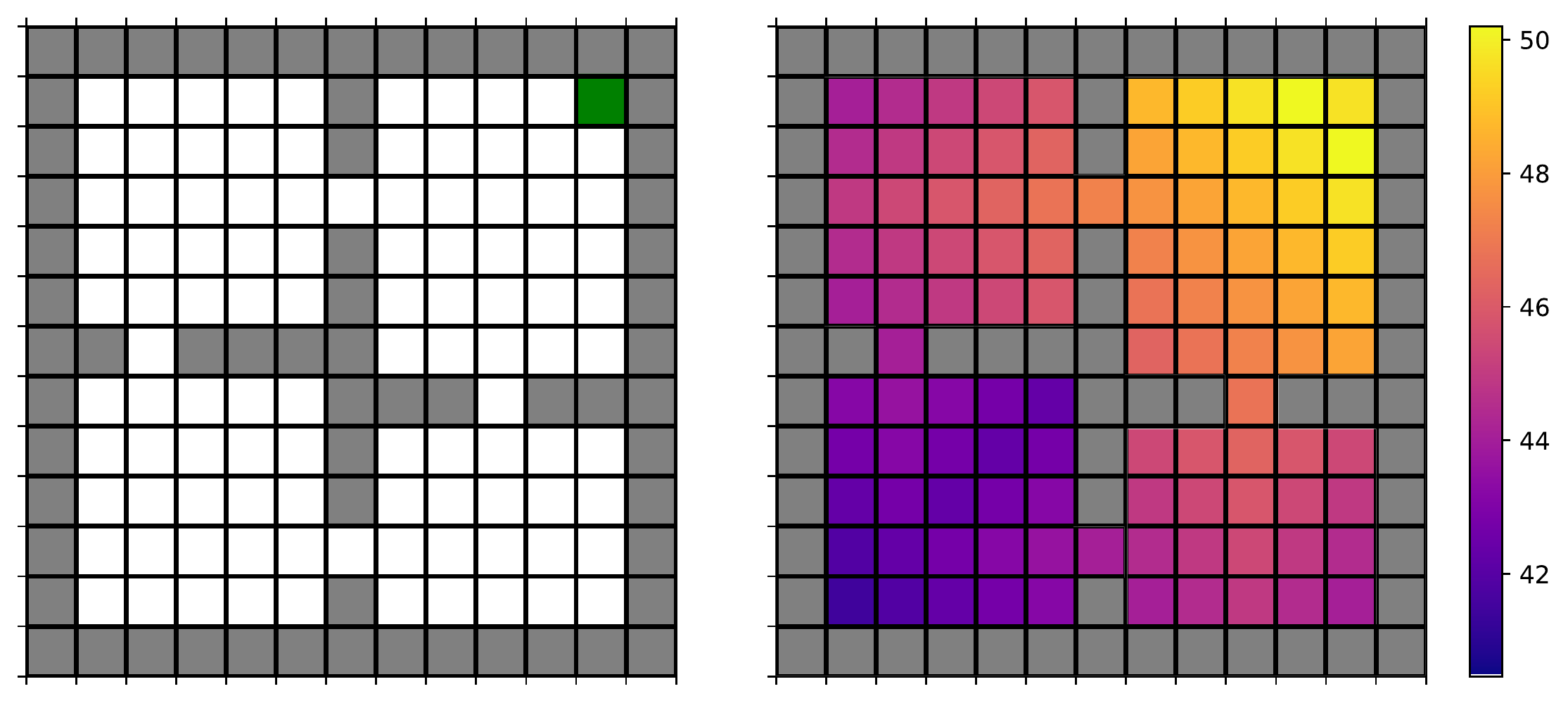}
  \caption{{\bf Left:} Four rooms domain with a single goal state (in green). {\bf Right:} Optimal values for each cell.}
  \label{fig:4rooms}
\end{figure}

\begin{figure*}[!h]
  \centering
  \includegraphics[width=\textwidth]{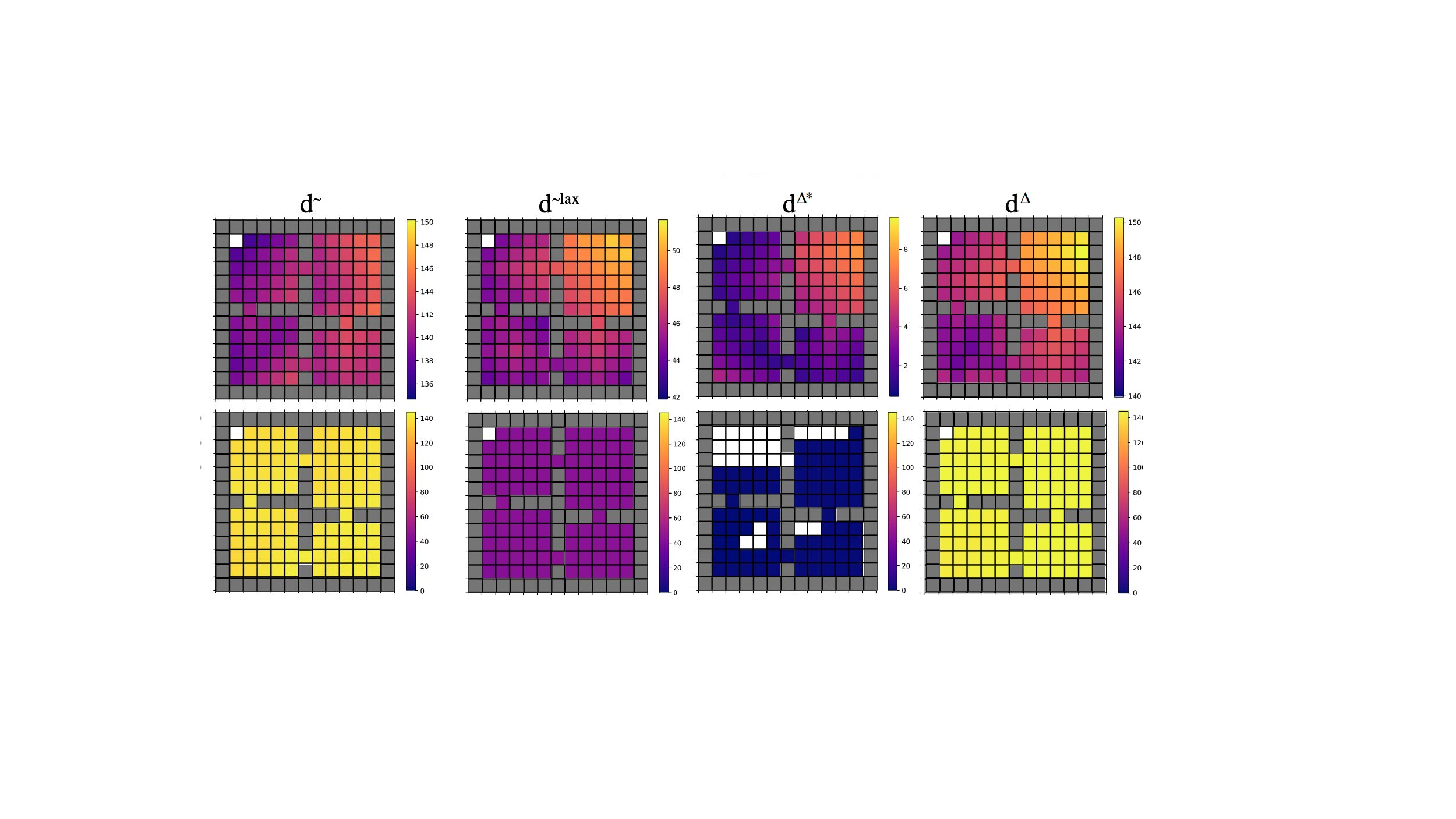}
  \caption{The top row illustrates the distances from the top-left cell to every other cell (note the color scales are shifted for each metric for easier differentiation between states). The bottom row displays $d(s, t) - |V^*(s) - V^*(t)|$, where $s$ is the top-left cell, illustrating how tight an upper bound the metrics yield on the difference in optimal values.}
  \label{fig:distances}
\end{figure*}

\begin{figure*}[!h]
  \centering
  \includegraphics[width=\textwidth]{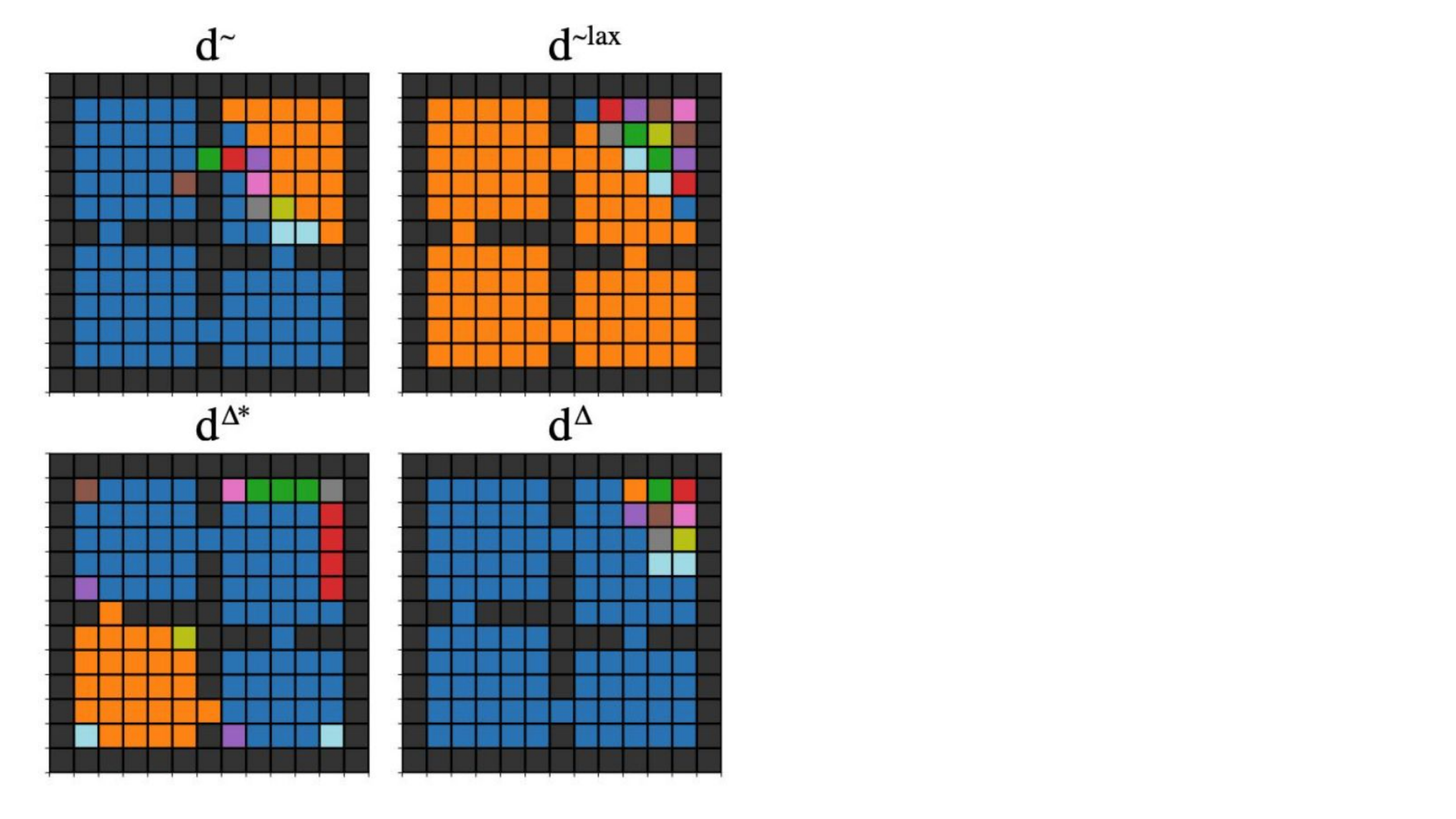}
  \caption{State clusters produced by the different metrics when targeting 11 aggregate states. There is no color correlation across metrics. }
  \label{fig:aggregation}
\end{figure*}

}

\end{document}